%% file: main.tex
\newif\ifprivate                
\newif\ifconf                   
\newif\ifshort                  
\newif\ifcompress               
\newif\ifuc                     
\newif\iffinal                  

\privatetrue\conffalse\shortfalse\compressfalse\uctrue\finalfalse

\documentclass[a4paper,12pt,twoside]{article} 
\usepackage[margin=1in]{geometry}
\usepackage{natbib,color}
\input{math_commands.tex}

\usepackage{hyperref}
\usepackage{float}
\usepackage{url}
\usepackage{algorithm}
\usepackage{algpseudocode}
\usepackage{ntheorem} 
\usepackage[capitalize]{cleveref} 
\crefname{figure}{Figure}{Figures} 

\ifprivate\def\todo#1{{\color[rgb]{.6,.3,.0}\it[[ToDo:\ #1]]}}\else\def\todo#1{}\fi 
\ifprivate\def\note#1{{\color[rgb]{.5,.5,.0}\it[[Note:\ #1]]}}\else\def\note#1{}\fi 

\usepackage[intoc,english]{nomencl}
\makenomenclature
\usepackage{etoolbox}
\renewcommand\nomgroup[1]{%
  \parskip=0ex\parsep=0ex\itemsep=0.5ex%
  \item[\bfseries
  \ifstrequal{#1}{A}{Agent-Environment Interaction}{%
  \ifstrequal{#1}{N}{Number sets}{%
  \ifstrequal{#1}{O}{ActivePTW}{%
  \ifstrequal{#1}{M}{PTW}{%
  \ifstrequal{#1}{P}{Probability and Statistics}{%
  }}}}}%
]}

\title{\ifconf\else\bf\Large\hrule height5pt \vskip 4mm\fi
Partition Tree Weighting for Non-Stationary Stochastic Bandits
\ifconf\else\vskip 4mm \hrule height2pt\fi }

\author{\centerline{\bf Joel Veness ~--~ Marcus Hutter ~--~ Andr\'as Gy\"orgy ~--~ Jordi Grau-Moya} \\
\centerline{Google DeepMind}\\
\centerline{\texttt{\{aixi,mhutter,agyorgy,jordigrau\}@google.com}}
}

\usepackage[all]{xy}
\usepackage{graphicx}

\newtheorem{theorem}{Theorem}
\newtheorem{definition}[theorem]{Definition}
\newtheorem{lemma}[theorem]{Lemma}

\newtheorem{proposition}[theorem]{Proposition}

\newenvironment{proof}{{\noindent\bf Proof.}}{\qed\vskip 1ex} 

\def\qed{\hspace*{\fill}\rule{1.4ex}{1.4ex}$\quad$\\} 
\def\fr#1#2{{\textstyle{#1\over#2}}} 
\def\frs#1#2{{^{#1}\!/\!_{#2}}} 
\def\lb{{\log_2}}               

\iffinal\ifconf\iclrfinalcopy\fi\fi 
\begin{document}

\maketitle

\begin{abstract}
This paper considers a generalisation of universal source coding for interaction data, namely data streams that have actions interleaved with observations.
Our goal will be to construct a coding distribution that is both universal \emph{and} can be used as a control policy.
Allowing for action generation needs careful treatment, as naive approaches which do not distinguish between actions and observations run into the self-delusion problem in universal settings.
We showcase our perspective in the context of the challenging non-stationary stochastic Bernoulli bandit problem.
Our main contribution is an efficient and high performing algorithm for this problem that generalises the Partition Tree Weighting universal source coding technique for passive prediction to the control setting.
\ifconf\else\vspace{5ex}\def\contentsname{\centering\normalsize Contents}\setcounter{tocdepth}{1}
{\parskip=-2.7ex\tableofcontents}\fi
\end{abstract}

\ifconf\centerline{\sc\large Keywords}\begin{quote}
\else\centerline{\bf\small Keywords}\begin{quote}\fi
stochastic bandits, non-stationary sources, source coding, control policy, self-delusions, partition tree weighting
\par\ifconf\end{quote}\else\end{quote}\fi \vskip 1ex

\section{Introduction}
\label{sec:intro}

This paper explores an alternate perspective on how to construct a type of universal agent from compression principles.
In contrast to popular decision-making frameworks such as reinforcement learning, which are built upon appealing to decision-theoretic notions such as Maximum Expected Utility, we instead construct an agent by trying to minimise the expected number of bits needed to losslessly describe general agent-environment interactions.
The appeal with this approach is that if we can construct a good universal coding scheme for arbitrary agent interactions, one could simply sample from this coding distribution to generate a control policy.
However when considering general agents, whose goal is to work well across \emph{multiple} environments, this question turns out to be surprisingly subtle.
Naive approaches which do not discriminate between actions and observations fail, and are subject to the self-delusion problem \citep{ortega2021shaking}.
In this work, we will adopt a universal source coding perspective to this question, and showcase its efficacy by applying it to the challenging non-stationary stochastic bandit problem.

In the passive case, namely, sequential prediction of observations under the logarithmic loss, there is a well developed universal source coding literature for dealing with non-stationary sources under various types of non-stationarity.
The most influential idea for modelling piecewise stationary sources is the transition diagram technique of \cite{Willems1996}. 
This technique performs Bayesian model averaging over all possible partitions of a sequence of data, cleverly exploiting dynamic programming to yield an algorithm which has both quadratic time complexity and provable regret guarantees.
Since then, a variety of more efficient algorithms \citep{Willems1997,Shamir1998,Hazan2009,Gyorgy2012,veness13,daniely15,cutkosky20a,zhao24} have been developed which trade off worst-case regret for computational efficiency, by dealing with only a subset of the possible partitions, while still maintaining regret guarantees for the entire class of possible partitions.
Of these techniques, arguably the technique of \cite{veness13} has the best complexity/performance trade-off, and this is the technique our work focuses on to extend to the case of control for non-stationary stochastic bandit problems.  

The non-stationary stochastic bandit problem is currently an active area of study, with many algorithms recently being proposed to address aspects of the problem.
A central challenge in this setting is to provide algorithms with good worst-case guarantees \emph{without} unrealistic knowledge of the change-point structure in advance.
There are a number of algorithms which adapt stationary stochastic bandit algorithms to incorporate ideas such as \emph{windowing} or \emph{discounting}.
Notable examples include Sliding Window/Discounted UCB \citep{Garivier2008}, as well as Thompson Sampling based variants \citep{Russo2017,Trovo2020}. 
These algorithms perform well when a good problem-dependent choice of hyperparameters is known in advance.
Recently, meta-algorithms built around stochastic restarting of base stationary bandit algorithms \citep{Wei2021} have shown much promise in addressing this problem in essentially a parameter-free way. 
Our new approach will be similar in spirit, but rather than stochastically restarting a base algorithm, we will adapt the techniques from the universal source coding literature to instead perform a type of Bayesian inference over a large class of possible restart configurations, essentially running many different restarting schedules in parallel efficiently.
As we shall see later, this will lead to an algorithm which can be considered to be a generalisation of the usual Thompson Sampling approach \citep{Thompson1933,Chapelle2011} for stationary stochastic bandit problems.
We investigate this technique both in terms of worst-case redundancy and experimentally, and show that the resulting algorithm compares favourably with the current state of the art.

Our paper is organised as follows: \cref{sec:background,sec:ucis} will introduce the notation we need to talk about bandit interactions using the language of universal source coding, as well as highlighting how the presence of the delusion problem makes this setting differ from the case of passive prediction; \cref{sec:setting} introduces the non-stationary stochastic bandit problem formally; \cref{sec:active_ptw} will derive and theoretically analyze our new algorithm, {\sc ActivePTW}; \cref{sec:results} will investigate our algorithm empirically, comparing and contrasting its performance to existing methods; and finally, we will finish with a discussion of results in \cref{sec:discussion} before concluding in \cref{sec:conclusion}.

\section{Background}
\label{sec:background}

We first give a brief overview of our notation for history based agents, which closely follows the recent textbook presentation of \cite{Hutter:24uaibook2}.

\paragraph{Interaction Data.}
A string $x_1x_2 \ldots x_n$ of length $n$ is denoted by $x_{1:n}$.
The prefix $x_{1:j}$ of $x_{1:n}$, $j\leq n$, is denoted by $x_{\leq j}$ or $x_{< j+1}$.
The empty string is denoted by $\epstr$.
The concatenation of two strings $s$ and $r$ is denoted by $sr$.
We will use $\epstr$ to denote the empty string, and use $\ell(s)$ to denote the length of a string $s$.
We let $\gA$, $\gO$, $\gR$ denote the finite action, observation and reward spaces respectively.
The percept space is defined by $\gE := \gO\times\gR$. 
A \emph{history} is a finite string of interaction data, namely 
$$h_{1:t} = a_1 o_1 r_1 a_2 o_2 r_2 \dots a_t o_t r_t \equiv a_1 e_1 a_2 e_2 \dots a_t e_t $$
for any $t \in \mathbb{N}$ and where each $a_i \in \gA$, $o_i \in \gO$, $r_i \in \gR$ and $e_i \in \gE$.
We will also use the shorthand notation $h_{<t} := h_{1:t-1}$.
The set of all possible histories will be denoted by $\gH$.
A policy is a function $\pi \colon (\gA \times \gX)^* \rightarrow (\gA \rightarrow [0,1])$ that maps a string of interaction experience to a distribution $\pi(\cdot \,|\, h_{1:t})$ over actions. 

\paragraph{Agent-Environment Interaction Loop.}
An \emph{agent} interacts with an \emph{environment} in cycles $t=1,2,\dots$.
In each cycle, an agent samples an action $a_t \sim \pi(\cdot \,|\, h_{<t})$ from its \emph{policy}, and communicates it to the environment.
The environment then sends a percept $e_t \sim \mu(\cdot  \,|\, h_{<t}a_t)$ to the agent, and the next cycle $t+1$ begins. 
A combination of a policy $\pi$ and an environment $\nu$ gives rise to an \emph{agent-environment measure} $\nu^\pi : \gH \to [0,1]$ whose conditional form is given by
\begin{equation*}
\nu^\pi(h_{t:m} \,|\, h_{<t}) := \prod_{k=t}^m \pi(a_k \,|\, h_{<k}) \nu(e_k \,|\, h_{<k} a_k). 
\end{equation*}
This can \emph{always} be decomposed into an \emph{environment measure} and \emph{policy measure} by respectively defining
\begin{equation*}
	\nu(e_{t:m} \,||\, a_{<t}) := \prod_{k=t}^m \nu(e_t \,|\, h_{<k} a_k) \text{~~and~~} \pi(a_{t:m} \,||\, e_{<t}) := \prod_{k=t}^m \pi(a_k \,|\, h_{<k}),
\end{equation*}
where the $||$ notation is used to make explicit that the right-hand side is given (equivalently, one can think of it as an intervention) and not conditioned on in the usual probabilistic sense.

\paragraph{Loss, Redundancy, Universality and Reward.}

The communication cost of encoding a history $h_{1:t}$ using an interaction measure $\nu^\pi$ is given by the code-length loss
\begin{equation*}
\gL(\nu^\pi, h_{1:t}) := -\lb \nu^\pi(h_{1:t}),
\end{equation*}
which can be interpreted as the number of bits needed to describe $h_{1:t}$ if $\nu^\pi$ was coupled to an arithmetic encoder \citep{Witten87}.
The competitor class we consider is defined by a set of interaction measures that are determined by a class of environments $\gM$. 
In particular, every $\mu \in \gM$ is coupled to an associated \emph{desired} policy $\pi_\mu$, yielding a competitor class of $\gU := \{ \mu^{\pi_\mu} \}_{\mu \in \gM}$. 
This motivates a natural notion of cumulative regret, which is also known as redundancy within the universal source coding literature.
\begin{definition}[Cumulative Regret]
The total regret at time $t$ of using an interaction measure $\rho^\pi$ instead of $\mu^{\pi_\mu}$ with respect to an interaction sequence $h_{1:t} \in (\gA\times\gE)^t$ is 
\begin{equation}\label{eq:regret_defn}
R^{\rho,\pi}_t(h_{1:t}) := 
\lb\frac{\mu^{\pi_\mu}(h_{1:t})}{\rho^\pi(h_{1:t})} 
= 
\lb \frac{\mu(e_{1:t} \,||\, a_{1:t})}{\rho(e_{1:t} \,||\, a_{1:t})} + \lb \frac{\pi_\mu(a_{1:t} \,||\, e_{1:t})}{\pi(a_{1:t} \,||\, e_{1:t})}.
\end{equation}
\end{definition}
Notice that \Eqref{eq:regret_defn} can be decomposed into two additive terms, which we will refer to as the \emph{environment redundancy} and the \emph{policy redundancy}. 
Although it is possible for $R^{\rho, \pi}_t(h_{1:t}) < 0$, the $\mu^{\pi_\mu}$-expected regret $\mathbb{E}_{\mu^{\pi_\mu}}[R^{\rho, \pi}_t(h_{1:t})] \geq 0$ due to the non-negativity of KL-divergences.
We say an interaction measure $\rho^\pi$ is \emph{universal} with respect to a set of interaction measures $\gU$ if for any $\mu^{\pi_\mu} \in \gU$, we have that $\mathbb{E}_{\mu^{\pi_\mu}}[R^{\rho, \pi}_t(h_{1:t})] = o(t)$. 

\section{Universal Coding of Interaction Sequences}
\label{sec:ucis}

For now we consider the case of coding a single interaction sequence.

\paragraph{Unknown environment, known policy.}

Let us now consider how to optimally encode an interaction sequence in the case where we knew the actions were generated by some choice of fixed policy $\pi$, but only knew that the true environment was sampled from a known distribution $w(\cdot)$ over $\gM$.
One can show in $\xi,\pi$-expectation that the optimal solution to code the percepts now is to use the predictive Bayes mixture distribution
\begin{equation}\label{eq:xicond}
{\xi}(e_t | h_{<t}a_t) := \sum_{\rho \in \gM} w^\rho_{t-1} \, {\rho}(e_t \,|\, h_{<t}a_t),~~\text{with}~~
w^{\rho}_{t-1} := 
\frac{w^{\rho}_0 \, \rho(e_{<t} \,||\, a_{<t})}{\sum_{\nu \in \gM} w^{\nu}_0 \, \nu(e_{<t} \,||\, a_{<t})}
\end{equation}
denoting the posterior belief in $\rho$ after seeing $t-1$ percepts, with prior $w^{\rho}_0 := w(\rho)$.
Even if the prior is unknown, and we use a prior $w'(\cdot)$ with the same support as $w(\cdot)$ in \Eqref{eq:xicond}, then one can show that the regret
\begin{equation*}
R^{\xi, \pi_\mu}_t(h_{1:t}) 
= \log \frac{\mu(e_{1:t} \,||\, a_{1:t})}{\xi(e_{1:t} \,||\, a_{1:t})}
\leq \log \frac{\mu(e_{1:t} \,||\, a_{1:t})}{w'(\mu) \mu(e_{1:t} \,||\, a_{1:t})}
= -\log w'(\mu)
\end{equation*}
In other words, in a regret sense, not much is lost due to not knowing $w(\cdot)$, and for example one can chose a uniform prior over $\gM$ if $|\gM|$ is finite to minimize the worst case regret.

\paragraph{Unknown environment, unknown  policy.}

Next we relax the assumption that we know the policy generating actions for our interaction sequence, and instead have to provide a universal policy $\hat{\pi}$, which intuitively gathers information about the true environment $\mu$ and adapts its behavior towards an environment-specific policy $\pi_\mu$.
To this end, we assume that we have a collection of policies $\Pi := \{ \pi_\rho \}_{\rho \in \gM}$, where $\pi_\rho$ specifies the desired behaviour for environment $\rho$, and go about constructing a universal coding scheme.
However, as discussed previously, we need to be careful to avoid the delusion problem and
respect the informational constraints implicit in the reinforcement learning setup.
A natural first attempt is to consider the policy which minimises the single-step $\xi$-expected code length of the actions 
\begin{equation*}
\hat{\pi}(a_t | h_{<t}) := \sum_{\rho \in \gM} w^\rho_{t-1} \pi_\rho(a_t | h_{<t})   
\end{equation*}
which \cite{Ortega2008} dubbed the \emph{Bayesian Control Rule}.
Interestingly, if $\pi_\rho$ is the expected utility-maximizing (often misleadingly called Bayes-optimal by others) policy for environment $\rho$,
then sampling $a_t$ from $\hat\pi$ is equivalent to Thompson sampling, which samples 
$\rho$ from the posterior $w^\rho_{t-1}$ and then $a_t$ from $\pi_\rho$.
Notice that the posterior probabilities in this formulation only use the information provided by the environment and not the agent.
Let's consider its regret,
\begin{align}
R^{\xi, \hat{\pi}}_t(h_{1:t}) &\leq  -\log w'(\mu) + \sum^t_{i=1}  \log \frac{\pi_{\mu}(a_i \,|\, h_{<i})}{\hat{\pi}(a_i \,|\, h_{<i})} \notag \\
&\leq
-\log w'(\mu) + \sum^t_{i=1}  \log \frac{\pi_{\mu}(a_i \,|\, h_{<i})}{w^\mu_{t-1}\pi_\mu(a_i \,|\, h_{<i})}\notag \\
&= -\log w'(\mu) - \sum^t_{i=1}  \log w^\mu_{t-1}.\label{eq:bcr_regret}
\end{align}
\Eqref{eq:bcr_regret} shows that provided 
$-\sum^t_{i=1} \log w^\mu_{t-1} = o(t)$, the Bayesian Control Rule is a universal source coding technique for interaction data.
The next section will build upon this perspective and use it to derive an algorithm for the non-stationary stochastic Bernoulli bandit problem (NSSBP).

\section{Problem Setting}
\label{sec:setting}

We first develop some notation for stationary stochastic Bernoulli bandits, which we then will generalise to a non-stationary setting.

\paragraph{Stationary stochastic Bernoulli bandits.}
A stationary stochastic Bernoulli bandit problem can be described by a tuple $(A, \Theta, \mu)$, where $A>1$ is the number of arms, with an associated parameter vector $\Theta \in [0,1]^A$.
The $i$th component of $\Theta$ will be denoted by $\theta_i$, which will determine the probability of observing a success when arm $i$ is pulled.
This gives rise to an environment measure $\mu$ that is defined by
$\mu(1 \,|\, h_{<t} a) := \theta_a$ and $\mu(0 \,|\, h_{<t} a) := 1-\theta_a$,
with action space $\gA := \{1, \dots, A\}$ and percept space $\gE := \{0,1\}$.

\paragraph{Non-stationary stochastic Bernoulli bandits.}
We first introduce some notation to describe temporal partitions.
A segment is a tuple $(c,d) \in \mathbb{N}\times\mathbb{N}$ with $c \leq d$.
A segment $(c,d)$ is said to overlap with another segment $(c',d')$ if there exists an $i \in \mathbb{N}$ such that $c \leq i \leq d$ and $c' \leq i \leq d'$.
A temporal partition $\gP$ of a set of time indices $\gS = \{ 1,2,\dots\,n \}$, for some $n\in\mathbb{N}$, is a set of non-overlapping segments such that for all $x\in\gS$, there exists a unique segment $(c,d) \in \gP$ such that $c \leq x \leq d$.

An abruptly changing non-stationary stochastic Bernoulli bandit problem (NSSBP) can be described by a tuple $\left(A, \gP, \{ \Theta^{c,d} \}_{(c,d)\in\gP}, \mu \right)$,
where $A > 1$ is the number of arms, $\gP$ is a temporal partition describing a change-point regime, and each $\Theta^{c,d} \equiv (\theta^{c,d}_1, \dots, \theta^{c,d}_A) \in [0,1]^A$ determining the probability of success $\theta^{c,d}_i$ of pulling arm $i$ at time $c \leq t \leq d$.
This gives rise to an environment measure $\mu$ that is defined by
\begin{equation}\label{eq:nssbp_liklihood}
\mu(e_{1:t} \,||\, a_{1:t}) := \prod_{(c,d)\in\gP} \prod_{t=c}^d 
\left[ e_t \theta^{c,d}_{a_t} + (1-e_t)(1-\theta^{c,d}_{a_t}) \right]
\end{equation}
with action space $\gA := \{1, \dots, A\}$ and percept space $\gE := \{0,1\}$.
Of course another way to look at this problem definition is a sequence of stationary stochastic bandit problems sharing a common action space.

\paragraph{Goal.}
Our goal will be to develop an algorithm which works well for a meaningful subset of the class of abruptly changing NSSBPs.
A general solution is of course impossible, since the change-point structure might be too rapid to allow for sufficient periods of exploitation, but in cases where the change-point structure is simple enough there is still hope for an effective algorithm, which is the subject of our next section.

\section{The ActivePTW Algorithm}
\label{sec:active_ptw}

This section adopts the perspective of \cref{sec:ucis} to develop a universal source coding technique for interaction data generated by an initially unknown environment measure $\mu$ associated with a choice of abruptly changing NSSBP $\left(A, \gP, \{ \Theta^{c,d} \}_{(c,d)\in\gP}, \mu \right)$.
We will first construct a universal environment measure for this problem class, and then turn our attention to constructing a universal policy which will work well for a meaningful subset of these problems.
Although this section aims to be self-contained, our work relies heavily on the Context/Partition Tree Weighting techniques \citep{Willems95,veness13}, and we highly encourage the interested reader to first study the original papers for a more complete understanding.

\subsection{Universal Environment Measure for NSSBPs}

We now construct our universal environment measure by an appropriate combination of the Krichevsky-Trofimov (KT) Estimator and the Partition Tree Weighting technique \citep{veness13}, two universal source coding techniques for passive (i.e.\ no actions or rewards, just observations) streams of data.
Our construction starts by defining a universal environment measure for stationary stochastic Bernoulli environments, which we shall call a \emph{KT Environment}, which we then generalise to the non-stationary case by combining with the Partition Tree Weighting technique.

\paragraph{KT Environment.}

We now describe a universal environment model for modelling the percepts received from any stationary stochastic bandit problem.
The main idea is to associate a distinct KT Estimator with each possible arm to model its unknown success probability.

A KT estimator can be obtained from a Bayesian analysis which combines a Binomial$(\theta)$ likelihood with a Jeffreys prior $w_J(\theta) := \pi^{-1} [\theta (1-\theta)]^{{-1/2}}$.
For all $n \in \mathbb{N}$, the KT-probability of a binary string $x_{1:n} \in \gX$ is defined as  
\begin{align*}
\text{\sc{kt}}(x_{1:n}) \equiv P_{KT}(\#0,\#1) &:= \int_0^1 w_J(\theta) \, \text{Binomial}_\theta(x_{1:n}) d\theta = \fr1\pi B(\#0+\fr12,\#1+\fr12) \\
~~~\text{where}~~~ \text{Binomial}_\theta(x_{1:n}) &:= \theta^{\#1}(1-\theta)^{\#0}
~~~\text{and}~~~ B(\alpha,\beta) = \Gamma(\alpha)\Gamma(\beta)/\Gamma(\alpha+\beta)
\end{align*}
with $\#0$ and $\#1$ denoting the number of 0s and 1s in $x_{1:n}$ respectively.
For the KT-estimator, it is known that its redundancy can be upper bounded by
\begin{equation}\label{eq:kt_redundancy}
-\lb \text{\sc{kt}}(x_{1:n}) + \lb \text{Binomial}_\theta(x_{1:n}) \leq \frac{1}{2} \lb(n) + 1
\end{equation}
for all $n\in\mathbb{N}$, all $x_{1:n} \in \{0,1\}^n$ and any $\theta \in [0,1]$.

Given the above, we can now define an environment measure
\begin{equation*}
\text{\sc kte}(e_{1:t} \,||\, a_{1:t}) := \prod_{a \in \gA} \text{\sc kt}(e^a_{1:t})
\end{equation*}
where $e^a_{1:t}$ is defined as the subsequence (potentially non contiguous) formed of the percepts $e_i$ such that $a_i = a$ for all $1 \leq i \leq t$.
The \text{\sc kte} probability can be maintained online using $O(1)$ time per step, with a fixed space overhead of $O(|\gA|)$ by using an $|A|$-sized array maintaining the sufficient statistics, i.e.\ the $a$ and $b$ counts, for each KT-estimator.

The following result shows that the KTE environment has bounded redundancy with respect to any sequence of arm pulls and any sequence of observations.

\begin{proposition}[KTE Redundancy]\label{prop:kte_redundancy}
For any stationary stochastic bandit problem $\left(A, \Theta, \mu \right)$, for all $n \in \mathbb{N}$, for all $e_{1:n} \in \gE^n$ and for all $a_{1:n} \in \gA^n$ we have
\begin{equation}\label{eq:kte_ub}
-\lb \text{\sc kte}(e_{1:t} \,||\, a_{1:t}) + \lb \mu(e_{1:t} \,||\, a_{1:t})	\leq  \frac{|\gA'|}{2}\lb \left(  \frac{n}{|\gA'|} \right) + |\gA'|,
\end{equation}
where $\gA':=\{a\in\gA:\ell(e^a_{1:t})>0\}$ is the set of all actions taken at least once.
\end{proposition}
\begin{proof}
We have that
	\begin{align*}
		& -\lb \text{\sc kte}(e_{1:t} \,||\, a_{1:t}) +  \lb \mu(e_{1:t} \,||\, a_{1:t})\\ 
		& \textstyle \stackrel{(a)}{=} -\lb \prod_{a \in \gA} \text{\sc kt}(e^a_{1:t}) + \prod_{a\in\gA} \lb \text{Binomial}_{\theta_a}(x^a_{1:n})\\
		& \textstyle \stackrel{(b)}{=} \sum_{a \in \gA'} \left[ -\lb  \text{\sc kt}(e^a_{1:t}) + \lb \text{Binomial}_{\theta_a}(x^a_{1:n}) \right] \\
		& \textstyle \stackrel{(c)}{\leq}  \sum_{a \in \gA'} \left[ \frac{1}{2} \lb( \ell(e^a_{1:t}) ) + 1 \right] \\
		& \textstyle \stackrel{(d)}{\leq} \frac{|\gA'|}{2}\lb \left(\frac{n}{|\gA'|}\right) + |\gA'|.
	\end{align*}
	Step (a) follows by definition; Step (b) upon rearranging; Step (c) from \Eqref{eq:kt_redundancy}; and Step (d) from maximising the worst case by exploiting the concavity of the logarithm.	
\end{proof}

\paragraph{PTW-KT Environment.}

Our PTW-KT environment is constructed hierarchically. 
A tree-based prior is first placed over a class of temporal partitions, and then, given a temporal partition, the parameters governing a segments behavior are modelled using a KTE.  
For computational reasons, we chose to use the class of \emph{binary temporal partitions}, which have some useful properties.

\begin{definition}[\cite{veness13}]
	Given a depth parameter $D \in \mathbb{N}$ and a time $t \in \mathbb{N}$, the set $\gC_D(t)$ of all binary temporal partitions from $t$ is recursively defined by
	\begin{equation*}
		\gC_D(t) := \bigl\{ \{ (t,t+2^{D}-1) \} \bigr\} \cup \left\{ \gS_1 \cup \gS_2 : \gS_1 \in \gC_{D-1} \left(t \right), \gS_2 \in \gC_{D-1} \left(t + 2^{D-1} \right)  \right\},
	\end{equation*}
	with $\gC_0(t) := \bigl\{  \{ (t,t) \} \bigr\}$.
	Furthermore, we define $\gC_D := \gC_D(1)$. 
\end{definition}

For example, $\gC_2 =$ $\bigl\{$ $\{ (1,4) \}$, $\{ (1,2), (3,4) \}$, $\{ (1,1), (2,2), (3,4) \}$, $\{ (1,2), (3,3), (4,4) \}$, $\{ (1,1), (2,2), (3,3), (4,4) \}$ $\bigr\}$.
Each binary temporal partition can be naturally thought of in terms of representing a tree structure which we will call a partition tree.
This class has some useful properties.
Given a set of segments $\mathcal{S}$, the set of time indices they cover is given by $\mathcal{T}(\mathcal{S}) := \bigcup_{(c,d)\in\mathcal{S}} \{ c, c+1, \dots, d \}$.
Any partition $\gP$ of consecutive time indices upto $2^D$ can be covered by some binary temporal partition $\gP' \in \gC_D$, in the sense that $\mathcal{T}(\mathcal{P}) = \mathcal{T}(\mathcal{P'})$ and every segment endpoint in $\mathcal{P}$ appears in $\mathcal{P}'$. 
Furthermore, there always exists a covering binary partition which contains no more than $|\gP| (\lceil \log n \rceil +1)$ segments \citep[Lemma~2]{veness13}.

Our PTW-KTE environment measure, for $t \leq 2^D$, can now be specified by 
\begin{equation}\label{eq:ptw_prior_kte_marginal}
	\text{\sc ptw-kte}_D(e_{1:t} \,||\, a_{1:t}) := \sum_{\gP\in\gC_D} 2^{-\Gamma_D(\gP)} \prod_{(c,d)\in \gP} \text{\sc kte}(e_{c:d} \,||\, a_{c:d}),
\end{equation}
where $\Gamma_D(\cdot)$ is a proper prior over binary temporal partitions within $\gC_D$.
Intuitively, one can think of $\Gamma_D(\gP)$ as a mapping which returns the description length of the tree structure associated with $\gP$ under a natural encoding, with smaller tree structures receiving higher weight. 
This particular weighting over tree structures was first introduced with the Context Tree Weighting method of \cite{Willems95}. 
Although the number of binary temporal partitions grows exponentially with $t$, it is possible to maintain both the marginal and predictive probabilities for $\text{\sc ptw-kte}_D$ with a space overhead of $O(D |\gA|)$ and time overhead of $O(D)$ for times $t < 2^D$. 

The next theorem upper bounds the redundancy of the environment measure $\text{\sc ptw-kte}_D$ with respect to an arbitrary NSSBP.

\begin{theorem}
 Given an abruptly changing NSSBP $\left(A, \gP, \{ \Theta^{c,d} \}_{(c,d)\in\gP}, \mu \right)$, for all $t \in \mathbb{N}$ such that $t \leq 2^D$, we have that the redundancy of $\text{\sc ptw-kte}_D$ with respect to the true environment $\mu$ is upper bounded by 
\begin{equation*}
-\lb \frac{\mu(e_{1:t} \,||\, a_{1:t})}{\text{\sc ptw-kte}_D(e_{1:t} \,||\, a_{1:t})} \leq 
|\gP|(\lceil \lb t \rceil + 1) \left[ \frac{|\gA|}{2}\lb  \left\lceil \frac{t}{|\gA||\gP|(\lceil \lb t \rceil + 1)} \right\rceil + |\gA| + 2
 \right]
\end{equation*}
\end{theorem}
for all $e_{1:t} \in \gE^t$, and for all $a_{1:t} \in \gA^t$.
\begin{proof}
This follows by a straightforward adaption of Theorem 1 in \cite{veness13}, and combining it with Proposition \ref{prop:kte_redundancy}.
\end{proof}

It is important to note that this bound is only meaningful in a predictive sense, since although it holds with respect to any given sequence of actions, it implies nothing about whether the true environment will be identified or not.
This point will become clearer in the next section when we come to specifying the {\sc ActivePTW} policy, as we shall see the benefits of forcing exploration.

\subsection{The ActivePTW Policy}

Following along the lines of \cref{sec:ucis} where we gave a universal source coding interpretation to the Bayesian Control Rule, our ActivePTW algorithm will sample from the policy which minimizes the expected single-step coding cost of an action at time $t$, where the expectation is taken with respect to the PTW-KTE posterior over environments.

\paragraph{Posterior Probability Density over NSSBP environments.}

The PTW-KTE prior probability over a NSSBP $\left(A, \gP, \{ \Theta^{c,d} \}_{(c,d)\in\gP}, \mu \right)$ is given by
\begin{equation*}
w_0(\gP, \{ \Theta^{c,d} \}_{(c,d)\in\gP}) = 2^{-\Gamma_D(\gP)} \prod_{(c,d)\in\gP}  \prod_{a\in\gA} w_J(\theta^{c,d}_a)
\end{equation*}
with associated likelihood function $\mu(\cdot \,||\, a_{1:t})$ as defined in \Eqref{eq:nssbp_liklihood}.
The posterior probability of an NSSBP after seeing $t > 0$ observations is
\begin{equation*}
w_t(\gP, \{ \Theta^{c,d} \}_{(c,d)\in\gP}) \propto w_0(\gP, \{ \Theta^{c,d} \}_{(c,d)\in\gP}) \, \mu(e_{1:t} \,||\, a_{1:t})   
\end{equation*}
with the marginal probability (i.e.\ normalising constant) given by $\text{\sc ptw-kte}_D(e_{1:t} \,||\, a_{1:t})$.
Furthermore, one can define a posterior over partitions by marginalising out the parameters governing the arms' behavior in each segment, giving
\begin{equation*}
w_t(\gP) \propto 2^{-\Gamma_D(\gP)} \prod_{(c,d)\in \gP} \text{\sc kte}(e_{c:d} \,||\, a_{c:d})
\end{equation*}
where the normalizing constant is given by \Eqref{eq:ptw_prior_kte_marginal}.
The posterior probability of a segment $(c,d)$
is thus given by the sum of the posterior probabilities of all partitions containing $(c,d)$, i.e.
\begin{equation}\label{eq:segment_posterior_sum}
w_t((c,d)) := \sum_{\gP' \in \{ \gP \in \gC_D \,:\, (c,d) \in \gP\}} w_t(\gP').   
\end{equation}
At any given time $t < 2^D$, there are exactly $D+1$ different \emph{active segments} due to the tree structured nature of the partitions, whose lengths are $2^0, 2^1, \dots 2^D$.
More formally, the set of active segments at time $t$ is defined by
\begin{equation*}
\text{\sc ActiveSegments}_D(t) := \{ (c,d) \in \mathbb{N}\times\mathbb{N} : c \leq t \leq d, \exists \gP \in \gC_D \text{~such that~} (c,d) \in \gP\}    
\end{equation*}
Note that {\sc ActiveSegments}$_D(t)$ can be computed in $O(D)$ time from the binary representation of $t-1$; see \cref{sec:active_segments} for details. 

We can now define our key quantity of interest, the posterior probability distribution over the active segments at time $t$, namely
\begin{equation*}
q_t((c,d)) := \begin{cases}
 w_t((c,d)) & \text{~if~} (c,d) \in \text{\sc ActiveSegments}_D(t);\\
 0 & \text{~otherwise.}
\end{cases}
\end{equation*}
We will subsequently show that this posterior probability always sums to 1 and can be computed in $O(D)$ time.
This will form the basis of an efficient computational realisation of the ActivePTW approach.

\paragraph{Computing the posterior over the active segments.}
\label{sec:EffPostStats} 

We first present a helpful identity which characterises the recursive structure of the marginal probability of the observations given a history, which we will exploit to efficiently compute the posterior weight of each active segment.

\begin{lemma}\label{lemma:ptw_recursive}
For depth $D \in \mathbb{N}$, given a history $h_t := a_1 e_1 \dots a_t e_t$ where $t \leq 2^D$, we have
\begin{multline*}
	\text{\sc ptw-kte}_D(e_{1:t} \,||\, a_{1:t}) = 
	\frac{1}{2} \text{\sc kte}(e_{1:t} \,||\, a_{1:t}) + 
	\\ \frac{1}{2} \text{\sc ptw-kte}_{D-1}(e_{1:k} \,||\, a_{1:k}) \, \text{\sc ptw-kte}_{D-1}(e_{k+1:t} \,||\, a_{k+1:t})
\end{multline*}
where $k = 2^{D-1}$ and $\text{\sc ptw-kte}_0(e \,||\, a) = \text{\sc kte}(e \,||\, a)$ by definition.
\begin{proof}
The proof is a straightforward adaptation of Lemma 1 in \cite{veness13}.
\end{proof}
\end{lemma}

From here onward, we assume that all relevant marginal probabilities of the form $\text{\sc ptw-kte}_i(\cdot \,||\, \cdot)$ are already known for $0 \leq i \leq D$, as they can be efficiently computed from \cref{lemma:ptw_recursive} and maintained within an array of size $O(D)$; see Section 3.3 of \cite{veness13} for details.

Now we show how to calculate the posterior weight $q_t\left((c,d)\right)$, $\forall(c,d) \in \text{\sc ActiveSegments}_D(t)$ efficiently, and show that for all $t \in \mathbb{N}$ we have
\begin{equation*}
\sum_{(c,d)} q_t\left((c,d)\right) ~=~ \sum_{(c,d) \in \text{\sc ActiveSegments}_D(t)\hspace{-9ex}} w_t\left((c,d)\right) = 1.    
\end{equation*}
First note that the set of active segments can be calculated in $O(D)$ time as per \cref{sec:active_segments}.
Now at any time $t$ there will be $D+1$ active segments of length $2^D$, $2^{D-1}$, \dots, $1$.  
Recall that the posterior weight for a given segment $(c,d)$ is defined by
\begin{equation*}
w_t((x,y)) := \sum_{\gP' \in \{ \gP \in \gC_D \,:\, (x,y) \in \gP\}} w_t(\gP'),   
\end{equation*}
in other words, the sum of all the posterior weights of the partitions in $\gC_D$ containing segment $(c,d)$.
Also, recall that the posterior weight of any given partition is
\begin{equation*}
w_t(\gP) ~=~ \frac{2^{-\Gamma_D(\gP)} \prod_{(c,d)\in \gP} \text{\sc kte}(e_{c:d} \,||\, a_{c:d})}{\text{\sc ptw-kte}_D(e_{1:t} \,||\, a_{1:t})},
\end{equation*}
with the normalising constant given by
\begin{equation*}
	\text{\sc ptw-kte}_D(e_{1:t} \,||\, a_{1:t}) := \sum_{\gP\in\gC_D} 2^{-\Gamma_D(\gP)} \prod_{(c,d)\in \gP} \text{\sc kte}(e_{c:d} \,||\, a_{c:d}).
\end{equation*}

Now let $s_i := (c_i, d_i)$ denote the active segment with length $2^i$ at an arbitrary time $t$ (whose index is suppressed as it is non-essential).
Note that $\sum_{\gP \in \gC_D} w_t(\gP) = 1$.
Now for any $\gP$, $s_i\in\gP$ for exactly one $i$,
hence $\gC_D^i:=\{\gP\in\gC_D:s_i\in\gP\}$ 
for $i\in\{0,\dots,D\}$ partitions $\gC_D$, i.e.\ 
$\bigcup_i \gC_D^i=\gC_D$ and $\gC_D^i\cap\gC_D^j=\emptyset$ for $i\neq j$. 
This implies 
\begin{equation*}
\sum_{(c,d)} q_t\left((c,d)\right) ~=~ \sum_i w_t\left(s_i\right)
~=~ \sum_i \sum_{\gP'\in\gC_D^i} w_t(\gP')
~=~ \sum_{\gP'\in\gC_D} w_t(\gP') ~=~ 1
\end{equation*}
hence $q_t(\cdot)$ is a valid prior over the active segments.
Lemma \ref{lemma:ptw_recursive} can be used inductively to efficiently compute the posterior weight for each active segment.
In particular we have that
\begin{equation*}
w_t(s_d) = \frac{1}{2} \text{\sc kte}(e_{c_d:d_d} \,||\, a_{c_d:d_d}) / \text{\sc ptw-kte}_D(e_{c_d:d_d} \,||\, a_{c_d:d_d})    
\end{equation*} 
and for $s_i$ with $0 < i < d$, one can show
\begin{equation*}
w_t(s_{i}) = [1- w_t(s_{i+1})] \times \frac{1}{2} \text{\sc kte}(e_{c_i:d_i} \,||\, a_{c_i:d_i}) / \text{\sc ptw-kte}_i(e_{c_i:d_i} \,||\, a_{c_i:d_i}),
\end{equation*}
with the base case $w_t(s_0) := [1- w_t(s_{1})]$.

\paragraph{Reference policies for known environments.}
To completely describe a universal class of agent-environment measures, we need to first specify a policy for each possible environment.
Because the number of possible NSSBPs is large, we do this analytically.
We consider the following two policies, which given an arbitrary  NSSBP $\left(A, \gP, \{ \Theta^{c,d} \}_{(c,d)\in\gP}, \mu \right)$, are defined as follows:
\begin{itemize}
    \item \emph{(Maximum Expected Utility)}
    Given a segment $(c,d)\in\gP$, at times $c \leq t \leq d$, the Maximum Expected Utility (MEU) policy is defined as
    \begin{equation*}
        \pi_\mu^*(a | h_{<t}) := \begin{cases}
            1/|\gA^*_t| & \text{if } a \in \gA^*_t;\\
            0 & \text{otherwise,}
        \end{cases}
    \end{equation*}
    where $\gA^*_t := \argmax_{a \in \gA} \{ \theta_a^{c,d} \}$.
    This policy simply plays an arm with the highest single-step $\mu$-expected return.
    \\
    \item \emph{(Maximum Expected Utility with Forced Exploration)}
    The Maximum Expected Utility with Forced Exploration (MEUFE) policy 
    \begin{equation*}
        \pi_\mu^\eps(a | h_{<t}) := \alpha(d-c)/|\gA| + \left(1-\alpha(d-c) \right) \pi_\mu^*(a | h_{<t})
    \end{equation*}
    explores uniformly at random with a segment length $l$-dependent probability $\alpha(l) := l^{-\tfrac{1}{2}}$, otherwise it behaves like the MEU policy.
    We shall see later (\cref{sec:results,sec:discussion}) that such forced exploration is essential to both the analysis and for certain challenging change-point regimes.
\end{itemize}

\paragraph{An efficient Bayesian Control Rule Policy for NSSBPs.}
Recall that the Bayesian Control Rule can be interpreted as a policy which if used as a coding distribution, would minimise the single step expected code length of the actions with respect to the agent's posterior distribution over possible environments.
This arguably underappreciated approach is already known to yield interesting algorithms in the stationary setting, as one can show that combining a greedy MEU reference policy with a KTE environment gives rise to the classic Thompson Sampling algorithm for stationary stochastic bandit problems \citep{Ortega2012AdaptiveCO}. 
The ActivePTW algorithm is derived from the same principles, only that now we consider a richer class of environments.
In particular, the BCR-type policy using the MEU reference policies now becomes
\begin{equation*}
\pi^*(a_t \,|\, h_{<t}) := \sum_{a \in \gA} w_t(\gP, \{ \Theta^{c,d} \}_{(c,d)\in\gP}) \, \pi^*_\rho(a_t \,|\, h_{<t}),
\end{equation*}
where $\rho$ is the associated environment measure with partition $\gP$ and set of arm parameters $\{ \Theta^{c,d} \}_{(c,d)\in\gP}$.
The BCR rule for the MEUFE policy class is defined similarly, and is obtained by replacing $\pi^*_\rho$ with $\pi^\eps_\rho$. 

The ActivePTW algorithm is best thought of as an efficient way of implementing a policy that samples actions from either $\pi^*(\cdot \,|\, h_{<t})$ or $\pi^\eps(\cdot \,|\, h_{<t})$.
The key insight is that although the number of possible partitions of time indices up to $t \leq 2^D$ is extremely large -- of the order of $O(2^{2^D})$ -- many of these partitions share common overlapping structure, so locally behave equivalently at any single fixed time.
Because of this equivalence, it is sufficient just to consider the posterior over the active segments, with the posterior mass associated to each active segment being the sum given in \Eqref{eq:segment_posterior_sum}.
This sum can be efficiently computed in $O(D)$ time and space by exploiting the recursive tree structure in the PTW prior; see \cref{sec:EffPostStats} for details.
This leads to a hierarchical sampling algorithm, where we first sample an active segment, condition on this and then sample from the Beta distributions which govern the posterior over each arm's latent Bernoulli parameters within the sampled segment.

\paragraph{Algorithm.}
We will present two equivalent versions of the {\sc ActivePTW} policy, the first for expository purposes in the main paper which focuses on the key conceptual steps (\cref{alg:active_ptw_simple}), and a more detailed presentation in \cref{sec:EffPostStats} which describes how to update the sufficient statistics for the PTW-KTE environment measure in time and space $O(D)$ and the sampling operations in time $O(D + A)$.
Structurally, the {\sc ActivePTW}$_D(\pi^*)$ algorithm is very similar to the classic Thompson Sampling algorithm for stationary environments, only with the additional step of first sampling an active segment from the PTW posterior.

\begin{algorithm}
\caption{ActivePTW$_D(\pi)$}\label{alg:active_ptw_simple}
\begin{algorithmic}
\\
\Require{Unknown environment $\mu$}
\Require{The number of arms $A \in \mathbb{N}$}
\Require{Maximum execution time $T \in \mathbb{N}$}
\Require{A choice of reference policy class $\pi_\rho \in \{\pi^*_\rho, \pi^\eps_\rho\}$}
\\
\Ensure{$T \leq 2^d$ for some $d \in \mathbb{N}$}
\\
\For{$t=1 \dots T$}
    \State $(x,y) \sim q_t(\cdot)$ \Comment{Sample an active segment}
    \State $\Theta \sim w_t(\cdot \,|\, (x,y))$ \Comment{Sample arm parameters for segment $(x,y)$}
    \State Set $\rho$ to an environment with matching $(x,y)$ and $\Theta$ 
    \State Perform action $a_t \sim \pi_\rho(\cdot \,|\, h_{<t})$ \Comment{Sample an action from reference policy $\pi_\rho$}
    \State Transmit action $a_t$ to the environment $\mu$.
    \State Receive percept $e_t \sim \mu(\cdot | h_{<t} a_t)$.
    \State $t \gets t + 1$
\EndFor
\\
\State \Return $a_{1:T} \in \mathcal{A}^T$
\\

\end{algorithmic}
\end{algorithm}

\paragraph{Source code.}
The reference implementation used to generate all of our subsequent empirical results is provided at:
\url{https://github.com/google-deepmind/active_ptw}.
The key implementation techniques are to manipulate all probabilities in log-space, and to compactly maintain the sufficient statistics for the posterior probabilities using an array of length $O(D|\gA|)$.

\subsection{Theoretical Analysis.}

\paragraph{Overview.}
Earlier we provided redundancy guarantees for the {\sc ptw-kte} environment measure, what remains is to provide a similar argument for for the agent measure.
In this section we restrict our attention to the $\Pi^\eps$ policy class, since only weak guarantees are possible for the MEU policies (see \cref{sec:discussion} for more details).

\paragraph{Behavior under Forced Exploration.}
We will start with a high probability quantification of the exploratory behaviour of the MEUFE policy on a segment of size $n=2^D$ with forced exploration rate $\alpha=1/\sqrt{n}$.
The probability any single arm is \emph{forcibly} explored at time $t$ is $p := \alpha/|\gA|$.
Let $B_{t,i} \sim \text{\sc Bernoulli}(p)$ denote the random variable who takes on value 1 if arm $i$ is forcibly explored at time $t$.
Next consider a collection of $|\gA|$ random variables $V_i = \sum_{t=1}^n B_{t,i}$ which count the number of times an arm $i$ is forcibly explored within a segment, with $\mathbb{E}[V_i] = \sqrt{n}/|\gA|$.
By Hoeffding's Inequality, we have that for $\eps>0$,
\begin{equation*}
\mathbb{P}\left(|V_i - \mathbb{E}[V_i]| \geq \epstr n \right) \leq 2\exp \{-2\eps^2 n\},  
\end{equation*}
and so via the union bound we have
\begin{equation*}
\mathbb{P}\left( \bigvee_{i=1}^A  |V_i - \mathbb{E}[V_i]|  \geq  \eps n \right )
\leq \sum_{i=1}^A \mathbb{P}\left(  |V_i - \mathbb{E}[V_i]|  \geq \eps n \right )
\leq 2 |\gA| \exp \{-2\eps^2 n\} ),
\end{equation*}
therefore with probability $1-2 |\gA| \exp \{-2\epstr^2 n\} )$ each arm will have been explored at least $\left \lfloor \sqrt{n}/|\gA|-\epstr n \right \rfloor$ times after $n$ draws from any sequence of MEUFE policies.

The next lemma, proven in \cref{sec:BetaTailBounds}, quantifies how the posterior over an arm's latent reward parameter concentrates rapidly around its true unknown value provided the arm is pulled sufficiently often.

\begin{lemma}[KT concentration]\label{lem:kt_concentration}
Given random variables $X_{1}, X_2, \dots, X_n$ with $X_i \sim \text{Bernoulli}(\theta)$, and letting $S_t := \sum_{i=1}^t X_i$, we have
\begin{equation*}
\mathbb{P}\left[ |\theta - \bar{\theta} \,|\, \geq \epsilon | S_t = s\right] \leq \frac{1.25(t-1)^{3/2}}{\sqrt{(s-\frs12)(t-s+\frs12)}}  
e^{-2((t-1)\epsilon^2)}
\end{equation*}
provided $2\leq s\leq t-2$, where $\bar{\theta} := \frac{\frs12+s}{t+1}$ denotes the posterior mean of the KT estimator.
\end{lemma}

Because of the fast rate of concentration within a segment, and provided the forced exploration of the MEUFE policy causes the posterior over active segments to concentrate sufficiently fast with respect to the latent change-point structure, its likely that in many cases of interest the {\sc ActivePTW} algorithm will enjoy low regret.
Making this precise is deferred to future work.

\paragraph{Improved Regret with Modified PTW prior.} 
It is possible to generalise the Context Tree Weighting \citep{Willems95} technique to permit non-uniform recursive hierarchical prior weightings.
The main idea is to replace the $\tfrac{1}{2}$ constants with $\gamma$ and $1-\gamma$ in Lemma \ref{lemma:ptw_recursive}, with $\gamma \neq \tfrac{1}{2}$.
This has the effect of biasing the shape of the learnt latent tree structure towards more parsimonious ($\gamma > \tfrac{1}{2}$) or more bushy ($\gamma < \tfrac{1}{2}$) tree structures.
As the PTW technique shares the same prior structure as CTW by definition, this idea directly leads to a generalisation of Lemma \ref{lemma:ptw_recursive}, where we have
\begin{multline}\label{eq:ptw_lemma_alpha}
	\text{\sc ptw-kte}_D(e_{1:t} \,||\, a_{1:t}) = 
	\gamma \, \text{\sc kte}(e_{1:t} \,||\, a_{1:t}) \,+ 
	\\ (1-\gamma) \, \text{\sc ptw-kte}_{D-1}(e_{1:k} \,||\, a_{1:k}) \, \text{\sc ptw-kte}_{D-1}(e_{k+1:t} \,||\, a_{k+1:t})
\end{multline}
for $D>0$.
In some situations this can be desirable; for example, in \cite{Tjalkens1993SequentialWA}, a known upper bound on the redundancy of the Multi-alphabet KT estimator (a type of Dirichlet-Multinomial model) is exploited to better adapt a CTW-based prior for multi-alphabet $k$-Markov sources.

Here we do something similar, noting that the worst case redundancy of the KTE environment has a dependence on the number of arms, suggesting that $\gamma$ should be set as a function of $A$.
\cref{eq:ptw_lemma_alpha} shows that the cost (in terms of bits) of not knowing the best in hindsight $\gP \in \gC_D$ can be calculated recursively.
There is one decision to make every time the recurrence is unrolled: either we \emph{stop}, in which case we pay a cost of $-\lb \gamma$; or we \emph{split} and pay a cost of $-\lb (1-\gamma)$.
Now, from \cref{prop:kte_redundancy}, we have that
\begin{equation*}
-\lb \text{\sc ptw-kte}_D(e_{1:t} \,||\, a_{1:t}) 
~\leq~ \frac{A}{2}\lb \left(  \frac{n}{A} \right) + A 
~\leq~ \frac{A}{2}\lb \left(  \frac{2^D}{A} \right) + A ~\leq~ A \left(\tfrac{D}{2}+1\right).
\end{equation*}
Ignoring the dependence on $D$, and noting that stopping is roughly $A$ times as costly as splitting, and that for any partition $\gP$ the number of splits $\approx$ the number of stops, we chose $\gamma = \argmin_\gamma\{ - A \lb \gamma - \lb (1-\gamma)\} = A/(A+1)$ to minimise the worst case redundancy independent of $\gP$.
Empirically, $\gamma := A/(A+1)$ performed substantially better than $\gamma := \tfrac{1}{2}$, and is used in all of our subsequent experiments in \cref{sec:results}.
Intuitively, as the number of arms gets larger, it is more difficult to learn each locally stationary environment, so more simple partition structures are to be favoured. 

\section{Experimental Results}
\label{sec:results}

For our results, we compared ActivePTW with a number of different algorithms from the non-stationary stochastic bandit literature, as well as some algorithms designed for the stationary setting as baselines:
\begin{itemize}
    \item \emph{Thompson Sampling} \citep{Chapelle2011}; here the environment is modelled by a collection of Beta-Bernoulli models, one per possible action. Each arm is played proportionally to the probability that it is best, which in practice means first sampling an environment $\rho$ from the posterior over environments, and then playing an arm with the optimal expected reward with respect to $\rho$.
    \item \emph{UCB} \citep{Auer2002}; this well known family of algorithms applies the optimism in the face of uncertainty principle, by estimating the expected value of playing each arm and greedily selecting an arm with the highest upper confidence bound. 
    The specific variant we investigate in this work is UCB1.
    \item \emph{Sliding Window UCB} \citep{Garivier2008};
    This UCB variant is similar to UCB1, except that it computes its value/upper confidence estimates from a fixed $W$ length window of the most recent data, which has the effect of continually forcing exploratory actions.
    \item \emph{MASTER} \citep{Wei2021}; MASTER is a recently introduced meta-algorithm which takes a base algorithm designed for the stationary case, and implements a clever stochastic restarting strategy based on the base algorithms known regret properties.
    In this work we combined MASTER with UCB1, as this was the best performing combination in our experiments.
    \item \emph{KL-UCB} \citep{Garivier2011}; the KL-UCB algorithm is a UCB variant which incorporates knowledge about the reward distribution; in our work we consider the case of Bernoulli distributed rewards.
    In our experiments, the additional MALG restarting tests never triggered, so the performance of MASTER and its internal worker algorithm MALG is identical.
\end{itemize}

\paragraph{Change-point regimes.}
We investigate three different change-point regimes, with varying number of actions and latent parameters.
All results are averaged across multiple episodes, along with 95\% confidence intervals (assuming asymptotic normality) reported directly in tabular form, or indicated graphically by shading.
The different regimes we consider are given below:
\\
\\
a) \emph{(Geometric spacing, uniform initialisation)}
Here each temporal partition consists of a set of segments, with each segment length $L \sim f_L(L=k) := p(1-p)^k $, with the success probability $p \in [0,1]$ governing how likely changes are to occur within an episode.
At the beginning of the episode, or when a change occurs, the Bernoulli parameters for each arm are sampled independently from a uniform distribution over $[0,1]$.
\\
\\
b) \emph{(No change-points, uniform initialisation)}
This setting is provided as a baseline.
Here the Bernoulli parameters for each arm are sampled independently from a uniform distribution over $[0,1]$ at $t=1$ and are held fixed for each episode.
\\
\\
c) \emph{(Geometric spacing, adversarial initialisation)}
This setting is identical to regime a), except that when a change occurs, the latent Bernoulli parameter for the previous best arm(s) stays the same.
This setting is more challenging from an exploration point of view.

\paragraph{Geometrically Distributed Change-points.}

\begin{figure}[t]
    \centering
    \includegraphics[scale=0.38]{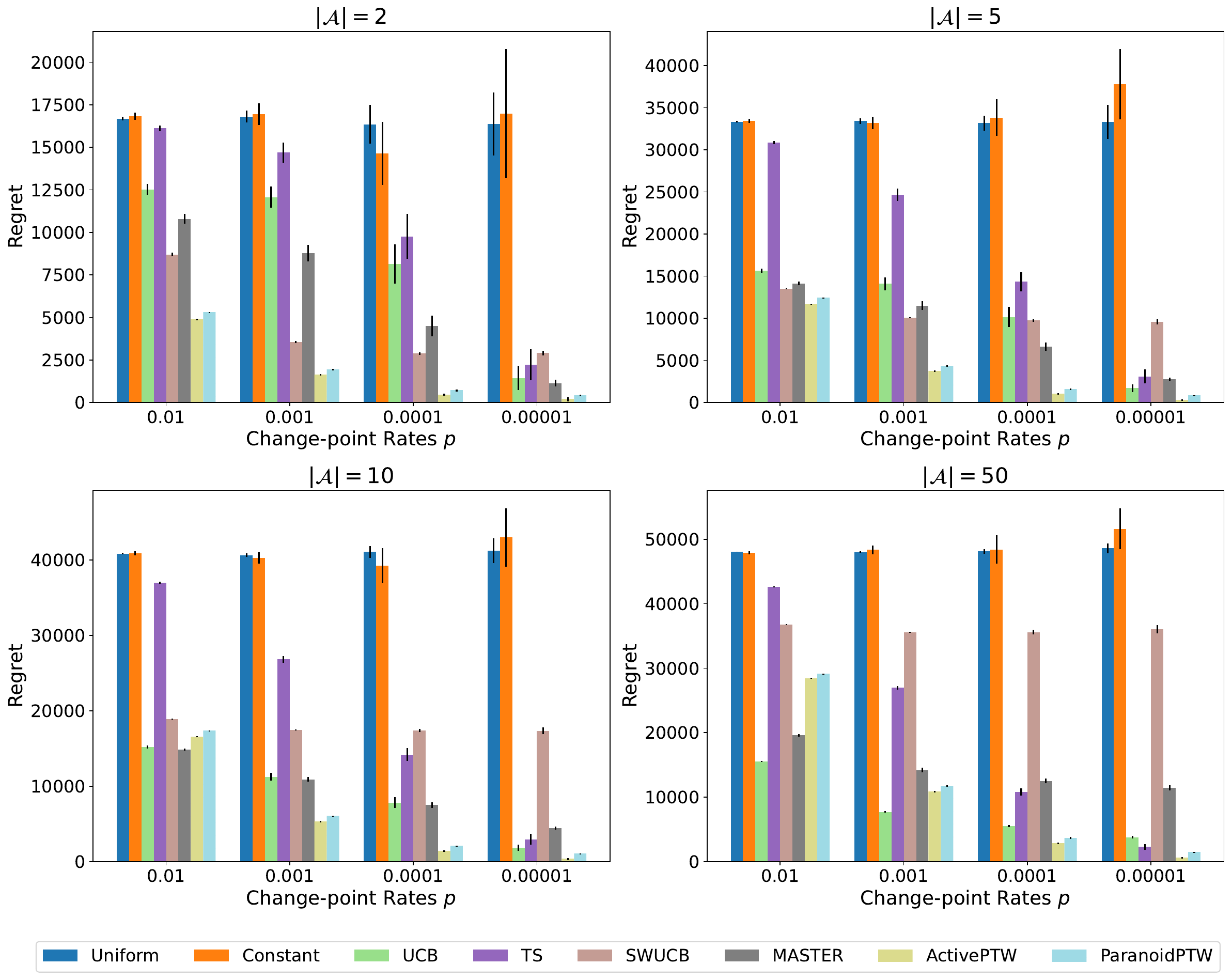}
    \caption{Detailed Results. Each panel shows the final regret for all algorithms under different change-point rates and action space cardinality.}
    \label{fig:geo_graph_summary}
\end{figure}

\cref{fig:geo_graph_summary} summarises the relative performance of each algorithm in change-point regime a).
\Cref{tbl:geo1,tbl:geo2,tbl:geo3,tbl:geo4}, provided in \cref{sec:detailed_results}, give the exact numbers and 95\% confidence intervals.
Here, the average empirical (cumulative) regret is reported for each algorithm/change-point rate for different action space sizes.
From inspection, one can see that the PTW variants generally outperform all other competing algorithms, with the except when the number of actions is large and the change-point rate high, then UCB outperforms all alternatives.
For Sliding Window UCB, these results are somewhat best-case, in the sense that the window size was set assuming that we knew the expected length of each segment in advance, i.e. $W=1/p$.

An illustrative example is shown in \cref{fig:geocpts}.
Here a single \emph{fixed} geometrically spaced segmentation (with success probability 0.0002) is sampled and used across all algorithms.
One can see the strong performance of the ActivePTW variants from the graphs, with MEU variant performing slightly best.
Here, the forced exploration is less necessary since pulling the previous best arm will still likely be informative.
The next best performer is MALG/MASTER, followed by Sliding Window UCB (whose window size was set to 5000). 

\begin{figure}
    \centering
    \includegraphics[width=\textwidth]{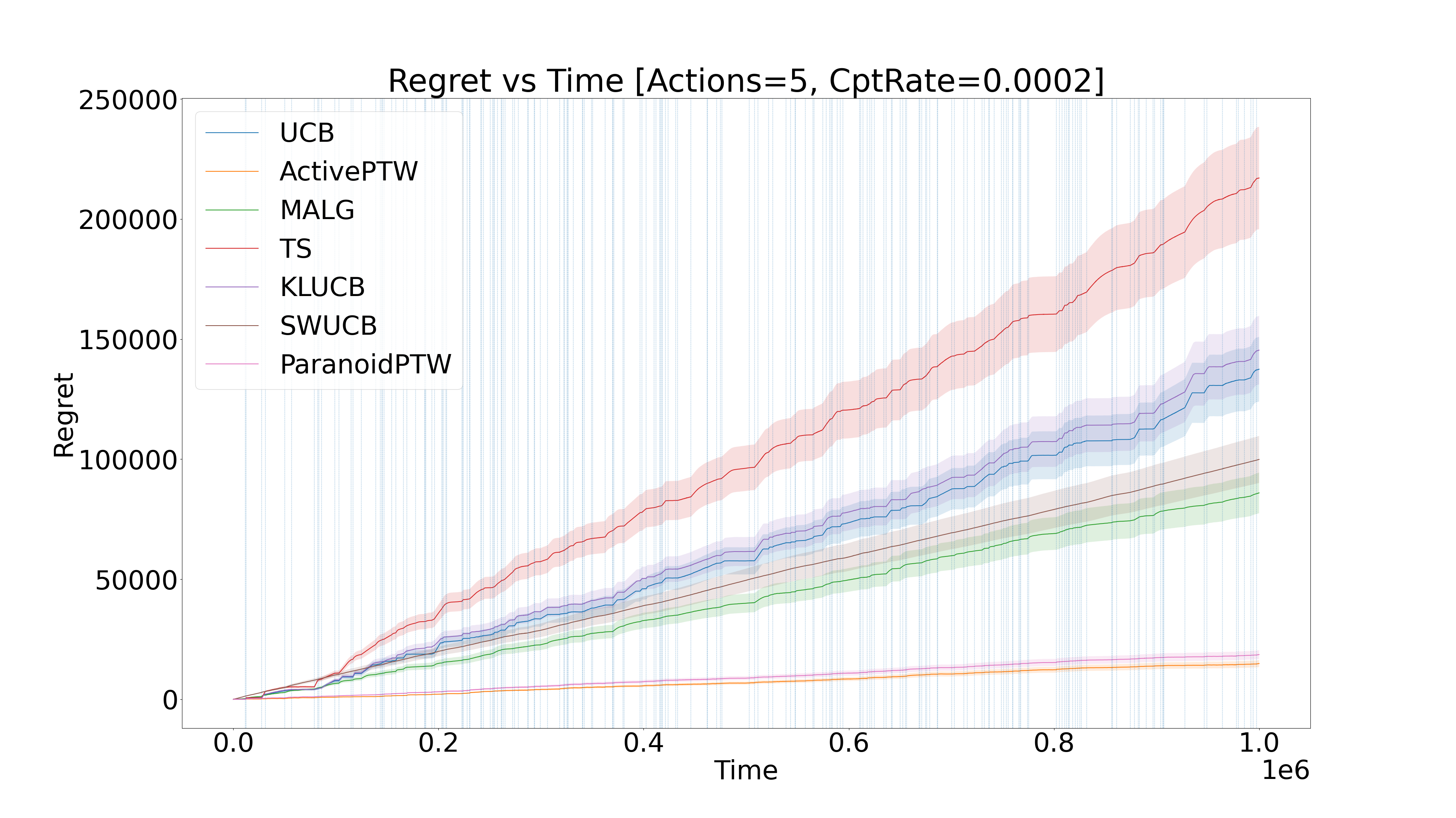}
    \vspace{-2.5em}
    \caption{Average regret across 400 repeats, approximate 95\% confidence intervals are indicated by shading, with $T=10^6$, $A=5$, and each segment length sampled from a geometric distribution with success probability $0.0002$.}
    \label{fig:geocpts}
\end{figure}

\paragraph{Stationary Environments.}
\label{sec:experiments_stationary}
\Cref{fig:stationary} shows the performance of both {\sc master} and the {\sc ActivePTW} algorithms versus some well known algorithms (UCB/KL-UCB/Thompson Sampling) designed for the stationary setting in change-point regime b).
Here the results for each algorithm were averaged across 400 episodes, $|\gA|=5$, with each episode consisting of 5000 steps.
To reduce variance, the same collection of seeds which govern the random draws for each environment were used to evaluate each algorithm (i.e.\ an application of the   Common Random Numbers variance reduction technique).

One can see that the performance of {\sc ActivePTW$(\pi^*)$} is nearly identical to that of Thompson Sampling.
This is the expected behaviour, as the environment redundancy of ${\text{\sc ptw-kte}_D(e_{1:t} \,||\, a_{1:t})}$ with respect to ${\text{\sc kte}_d(e_{1:t} \,||\, a_{1:t})}$ is upper bounded by $-\log \gamma$ bits for any $t \leq 2^D$ using a standard dominance argument. 
What happens in practice is that the posterior mass concentrates on partitions with only one segment covering the entire range of the observed data, so the algorithm essentially collapses to staionary Thompson Sampling in this case.
The performance of {\sc ActivePTW}$(\pi^\eps)$ is predictably worse than {\sc ActivePTW$(\pi^*)$}, as the forced exploration is unnecessary in this simpler setting, though it is important to note that the empirical performance here is substantially better than MASTER.
KL-UCB also unsurprisingly outperforms UCB in this setting, as by design it has additional knowledge about the structure of the reward distribution.

\begin{figure}
    \centering
    \vspace{-1em}
    \includegraphics[width=1\textwidth]{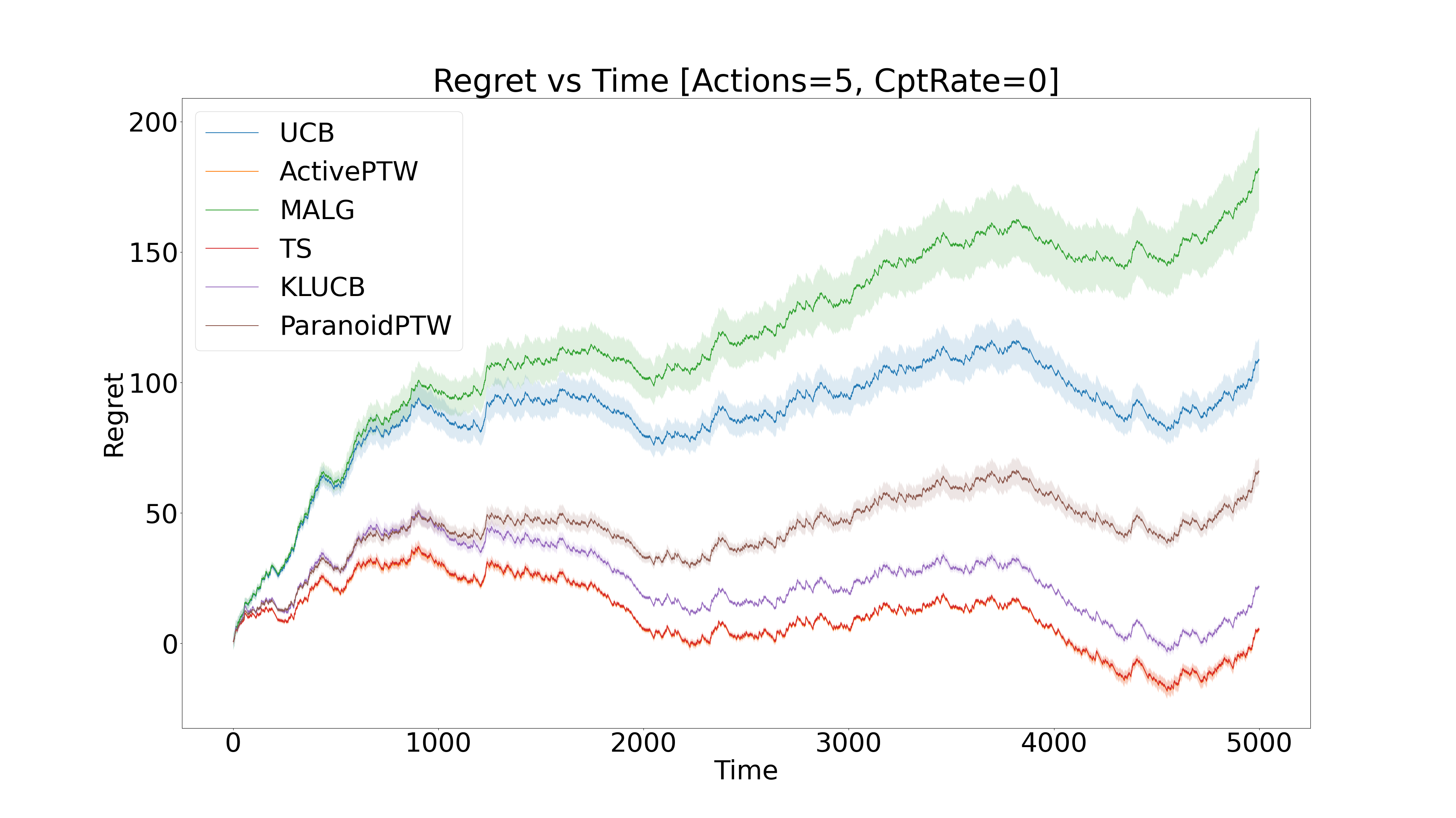}
    \vspace{-1em}
    \caption{A baseline comparison on stationary stochastic bandit problems. Approximate 95\% confidence intervals are indicated by shading, and are computed using 400 runs of each algorithm. Note that the performance of Thompson Sampling and ActivePTW using an MEU reference policy is nearly identical and difficult to distinguish on the graph.}
    \label{fig:stationary}
\end{figure}

\paragraph{An adversarial case that motivates the need for forced exploration.}
\label{sec:experiments_tricky}

As we saw previously, when using the MEU reference policy, the performance of {\sc ActivePTW} matches that of Thompson Sampling in stationary environments.
This is unsurprising given the design of the algorithm, but sadly undesirable in terms of worst case performance.
In particular, Theorem 31.2 in \cite{Lattimore2020Book} gives a constructive proof\footnote{The result is given for abruptly changing Gaussian bandits, but a similar construction and result applies to the Bernoulli case.} that any algorithm which enjoys logarithmic regret (such as Thompson Sampling) in stationary stochastic bandit problems must suffer regret $\Omega(t/\log(t))$ on some NSSBP instances. 

Taking inspiration from the aforementioned theoretical construction, here we considered an NSSBP with two equally sized segments, each of length $5000$, and set the number of arms to $A=10$.
In the first segment, $\theta^{1,2500}_i = 0.1$ for all $1 < i \leq A$, with $\theta^{1,2500}_1=0.2$.
In the second segment, $\theta^{2501,5000}_i = 0.2$ for all $i \neq 2$, and $\theta^{2501,5000}_2 = 0.8$.
The important thing to note is that the expected return of playing the best arm (arm 1) in segment 1 remains the same in segment 2, and that only one arm in segment 2 (arm 2) has an expected return better than the best arm in segment 1.
This implies that playing arm 1 in segment 2 provides no information that a change-point has occurred.
\cref{fig:adversarial} shows a summary of results in this setting, averaged over 1600 runs.
One can see that {\sc ActivePTW} using a MEU reference policy performs poorly on this hard example, essentially performing similarly to naive Thompson Sampling, and being outperformed by both MASTER and {\sc ActivePTW}$(\pi^\eps)$.
This is a concrete example which justifies the need for forced exploration.

Note that there are regimes where a more exploitative reference policy would be expected to perform better; in particular, if whenever a change-point occurred, if the previously optimal arm's latent Bernoulli parameter changed by some non-trivial constant amount, then the change-point detection problem is very similar to the passive case of sequence prediction in which the Partition Tree Weighting technique is known to work well \citep{veness13}, so here {\sc ActivePTW}$(\pi^*)$ would also be expected to perform well.

\begin{figure}[ht]
    \centering
    \vspace{-1em}
    \includegraphics[width=1\textwidth]{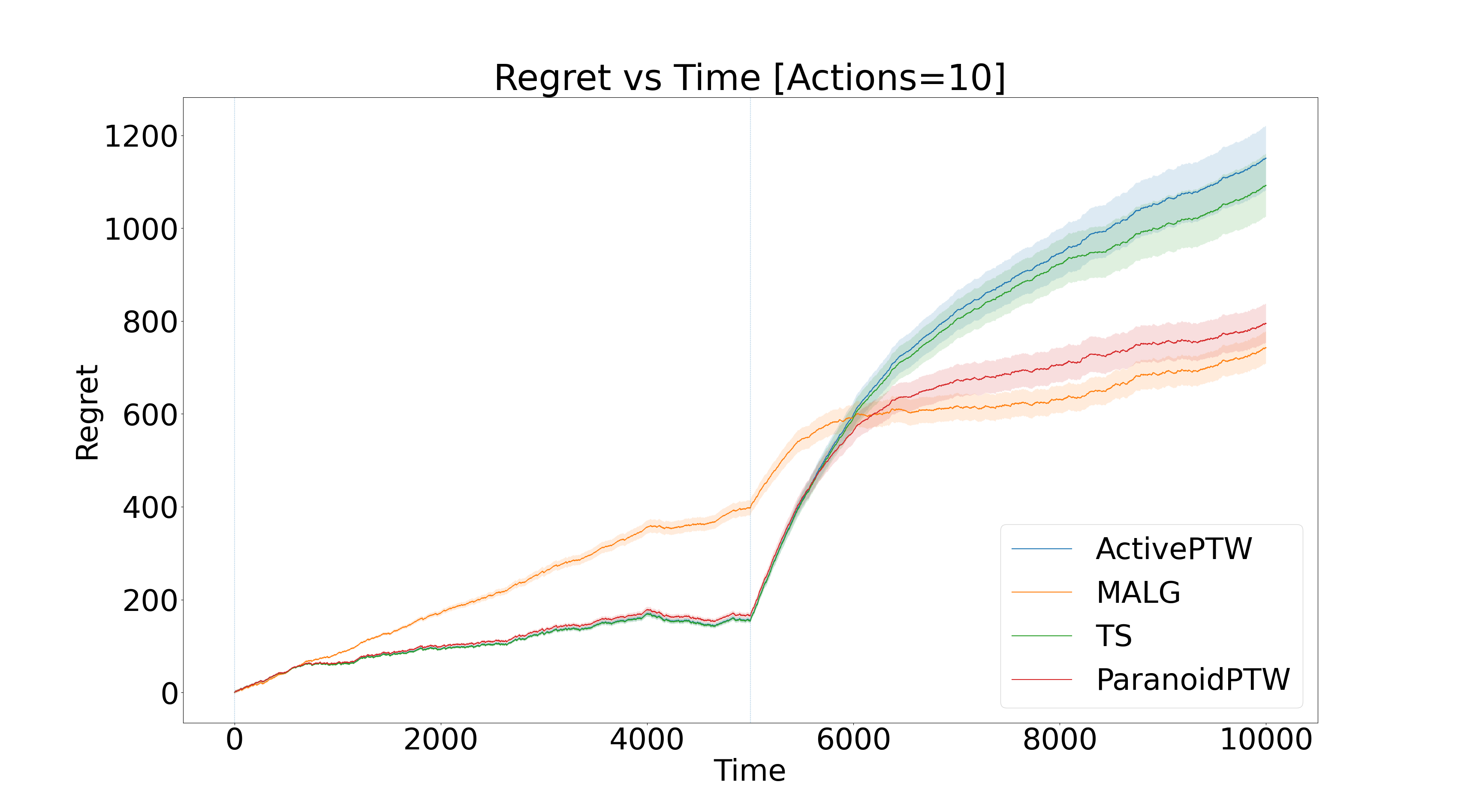}
    \vspace{-2em}
    \caption{An adversarial example, with a single changepoint at $t=5000$. Approximate 95\% confidence intervals are indicated by shading, and are computed using 1600 runs of each algorithm.}
    \label{fig:adversarial}
\end{figure}

\section{Discussion}
\label{sec:discussion}

\paragraph{Connection with Thompson Sampling.}

By identifying $\pi_\rho$ with the maximum expected utility policy $\pi_\rho^*$ for each $\rho \in \mathcal{M}$, instantiating the Bayesian Control Rule and sampling $a_t \sim \hat{\pi}$ gives the now well known Thompson Sampling for stochastic bandits algorithm which has strong finite time regret guarantees \citep{kaufmann2012thompson}.

The way in which the Bayesian Control Rule explores is interesting -- at each time step, it is limited to a convex combination of the $\pi_\rho(\cdot \,|\, ax_{1:t})$ for $\rho \in \mathcal{M}$.
Even if the environments are distinguishable in principle, this class of policies may not be sufficiently expressive to allow for the posterior to concentrate in longer horizon settings.
To this end, more general forms of Thompson Sampling for general agents acting over multi-step horizons have been considered by \cite{Strens2000} and \cite{Leike2016}, with the later work using a growing horizon and a recoverability assumption to enable asymptotic optimality guarantees to be provided for general environments.
Another related work for general environments, which explores more aggressively, is the BayesEXP \citep{Lattimore2014BayesianRL} algorithm, which optimises the expected multi-step information gain across essentially \emph{all} policies. 

\paragraph{Connection with the origins of cybernetics.}
In his classic philosophical work,``The Human Use of Human Beings", \cite{Wiener1950} put forth the idea that a living organism (or agent) could be viewed as a type of entropy-constrained adaptive process coupled to an input/output channel, giving rise to his cybernetic model.
This shares similarities with modern AI frameworks such as reinforcement learning which make a sharp distinction between agent and environment, but importantly differs in that there are no rewards, nor any commitment to various normative notions to constrain what so-called rational behavior looks like. 
This perspective is very much in the spirit of the universal source coding perspective adopted in this paper, which in turn builds on the active-coding work of \cite{Ortega2008, Ortega2009, Ortega2010AnAF, Ortega2012AdaptiveCO, ortega14}.
In light of this, it seems that the Bayesian Control Rule can be interpreted as just one, albeit fundamental, example of such a universal control policy, of which the classic Thompson Sampling technique for stationary stochastic bandits can be derived by further constraining the class of reference policies considered.

\section{Conclusion}
\label{sec:conclusion}

We presented a new algorithm, {\sc ActivePTW}, for the non-stationary stochastic bandit problem.
Our work differs from the design of many bandit algorithms in that it is derived from an application of universal source coding principles to the case of control.
It does not seem unreasonable to expect that other universal source coding techniques could be similarly generalised to yield different types of agents with interesting properties.
Given the strength of the empirical evidence supporting both our work and Thompson Sampling in the stationary case, this perspective seems like a promising avenue for further investigation.

\subsubsection*{Acknowledgments}
We would like to thank Pedro Ortega for many enjoyable discussions which inspired this work, as well as Nevena Lazic for sharing some code which helped us verify some of our empirical results.

\bibliographystyle{plainnat}
\bibliography{refs}

\clearpage\appendix


\nomenclature[A]{\(\mathcal{A}\)}{Action space}
\nomenclature[A]{\(\mathcal{O}\)}{Observation space}
\nomenclature[A]{\(\mathcal{R}\)}{Reward space}
\nomenclature[A]{\(\mathcal{E}\)}{Percept space $\equiv$ Cartesian product of observation and reward spaces $\equiv \mathcal{O}\times\mathcal{R}$}
\nomenclature[A]{\(h_{1:t}\)}{A history $h_{1:t} := a_1 e_1 a_2 e_2 \dots a_t e_t  \in \mathcal{H}$}
\nomenclature[A]{\(\mathcal{H}\)}{Set of all possible histories}
\nomenclature[A]{\(t\)}{time index, $t \in \mathbb{N}$}
\nomenclature[A]{\(e_{1:t}\)}{A percept history where $e_{1:t} := e_1  e_2 \dots e_t  \in \mathcal{E}^t$}
\nomenclature[A]{\(a_{1:t}\)}{An action history where $a_{1:t} := a_1  a_2 \dots a_t  \in \mathcal{A}^t$}
\nomenclature[A]{\(\mu\)}{True environment, typically latent/unknown}

\nomenclature[M]{\((c,d)\)}{A segment, a set of times indices $\{ c, c+1, \dots, d \}$}
\nomenclature[N]{\(\mathbb{R}\)}{Real numbers}
\nomenclature[N]{\(\mathbb{N}\)}{Natural numbers}
\nomenclature[N]{\([a,b]\)}{Closed real interval between $a$ and $b$}

\nomenclature[P]{\(w_J(\theta)\)}{Jeffrey's Prior over $\theta \in [0,1]$ for a Beta-Binomial model}
\nomenclature[P]{\(\text{Binomial}_{\theta}(\cdot)\)}{Binomial distribution with parameter $\theta$}

\nomenclature[O]{\(A\)}{Number of arms}
\nomenclature[O]{\(\gamma\)}{Tree shape hyper-parameter for the hierarchical prior over partition trees. $\gamma$ is the recursive stop probability; ($1-\gamma$) is the recursive split probability.}
\nomenclature[M]{\(\text{$D$}\)}{Depth parameter defining maximum time $t < 2^D$}
\nomenclature[O]{\(\text{\sc ptw-kte}_D\)}{ActivePTW meta-algorithm}
\nomenclature[O]{\(\text{\sc KTE}\)}{KT Environment}
\nomenclature[O]{\(\text{\sc KT}\)}{Krichevsky–Trofimov estimator}
\nomenclature[O]{\(\text{$\pi_\mu^*$}\)}{Maximum Expected Utility reference policy (for a horizon of 1) for environment $\mu$}
\nomenclature[O]{\(\text{$\pi_\mu^\eps$}\)}{Maximum Expected Utility reference policy with forced exploration}
\nomenclature[O]{\((A,\Theta,\mu)\)}{A stationary stochastic Bernoulli bandit problem}
\nomenclature[O]{\(\Theta\)}{The latent parameters $\Theta \in [0,1]^A$ for the reward distribution of a stationary stochastic bandit problem}
\nomenclature[O]{\(\left(A, \gP, \{ \Theta^{c,d} \}_{(c,d)\in\gP}, \mu \right)\)}{A non-stationary stochastic bandit problem}
\nomenclature[M]{\(\gP\)}{A set of segments which partitions times 1 to $2^D$}
\nomenclature[M]{\(\gC_D\)}{The set of binary temporal partitions of times 1 to $2^D$}
\nomenclature[O]{\(\Theta^{(c,d)}\)}{The latent parameters $\Theta \in [0,1]^A$ for the reward distribution for segment $(c,d)$}

\printnomenclature

\clearpage
\section{Beta Tail Bounds}\label{sec:BetaTailBounds} 

\def\fr#1#2{{\textstyle\frac{#1}{#2}}} 
\def\eps{\varepsilon}           
\def\KL{\text{KL}}

\paragraph{Definitions.}
For $x,t\in[0;1]$ and $a,b>0$ we define
\begin{align*}
  \text{Beta density:} & & \text{Beta}(t;a,b) ~&:=~ \frac{t^{a-1}(1-t)^{b-1}}{B(a,b)} \\
  \text{Incomplete Beta function:} & & B(x;a,b) ~&:=~ \int_0^x t^{a-1}(1-t)^{b-1}dt \\
  \text{(Complete) Beta function:} & & B(a,b) ~&:=~ B(1;a,b) ~=~ \frac{\Gamma(a)\Gamma(b)}{\Gamma(a+b)} \\
  \text{Cumulative Beta distribution:} & & I_x(a,b) ~&:=~ \frac{B(x;a,b)}{B(a,b)} ~=~ \int_0^x  \text{Beta}(t;a,b)dt
\end{align*}

\paragraph{Beta posterior asymptotics.}
Consider $X_i\sim\,$Bernoulli($\theta$) and prior $p(\theta)=\text{Beta}(\theta;\alpha,\beta)$.
with $\alpha,\beta>0$ and $X_+:=X_1+...+X_n$. Then the posterior is 
$$
  p(\theta|X_+=k) ~=~\text{Beta}(\theta;k+\alpha,n-k+\beta)
$$
which has mean $\mu:=(k+\alpha)/(n+\alpha+\beta)$ 
and variance $\sigma^2=\mu(1-\mu)/(n+\alpha+\beta+1)$
and is asymptotically normal for $k,n\to\infty$, hence
\begin{equation}\label{eq:betaasympt}
  P[|\theta-\mu|\geq\eps|k;n,\alpha,\beta] ~\approx~ 2\Phi(-\eps/\sigma) 
  ~\approx~ \sqrt{\fr{2\sigma^2}{\pi\eps^2}}\cdot\text{\large{e}}^{-\eps^2/2\sigma^2} 
  ~=~ \text{\large{e}}^{-\eps^2\Theta(n)}
\end{equation}
by asymptotic normality and the asymptotic expansion of the cumulative normal function $\Phi$.
Iff $\eps=o(n^{-1/3})$ / $\eps=\omega(n^{-1/2})$ the first/second $\approx$ 
means asymptotically equal (ratio tends to 1, otherwise the ratio is exponentially/polynomially off).
E.g.\ $\eps=\Theta(n^{-2/5})$ satisfies both, but the second approximation can easily be improved (and a polynomial factor is is not severe anyway).
The above asymptotic expression informs us about the best achievable rigorous bounds:

\paragraph{Beta tail bounds.}
For $a,b>0$, the maximum of $t^a(1-t)^b$ is attained at $t=\hat\theta:=\fr{a}{a+b}$,
and $t^a(1-t)^b$ is monotone increasing/decreasing before/after this point.
Therefore, for $x\leq\hat\theta$ we have
\begin{align*}
  B(x;a+1,b+1) ~&=~ \int_0^x t^a(1-t)^b dt
  ~\leq~ x\cdot x^a(1-x)^b ~=~ x\cdot e^{-(a+b)[\KL(\hat\theta||x)+H(\hat\theta)]} \\
  \text{where}~~~ \KL(p||q) ~&:=~ ~~p\ln\fr{p}{q}+(1-p)\ln\fr{1-p}{1-q} ~~~~~~~\text{is the binary KL-divergence, and} \\
  \text{and}~~~~~~~~ H(p) ~&:=~ -p\ln p-(1-p)\ln(1-p) ~~~\text{is the binary entropy.}
\end{align*}
The equality can easily be verified by substitution.
A lower bound for $B(a+1,b+1)$ can be obtained via Stirling bounds
$$
  \sqrt{2\pi n}(\fr{n}{e})^n e^{\frac{1}{12n+1}} ~<~ \Gamma(n+1) 
  ~<~ \sqrt{2\pi n}(\fr{n}{e})^n e^{\frac{1}{12n}}
$$
which is valid for real $n\geq 1$. Hence for $a,b\geq 1$ we have 
\begin{align*}
  B(a+1,b+1) ~&=~ \frac{\Gamma(a+1)\Gamma(b+1)}{(a+b+1)\Gamma(a+b+1)} \\
  ~&\geq~ \frac{ \sqrt{2\pi a}(\fr{a}{e})^a e^{\frac{1}{12a+1}} \cdot
                 \sqrt{2\pi b}(\fr{b}{e})^n e^{\frac{1}{12b+1}}}%
                 {(a+b+1)\sqrt{2\pi(a+b)}(\fr{a+b}{e})^{a+b} e^{\frac{1}{12(a+b)}}} \\
  ~&=~ \frac{1}{(a+b+1)} \sqrt{\frac{2\pi ab}{a+b}} \Big(\frac{a}{a+b}\Big)^a \Big(\frac{a}{a+b}\Big)^b \cdot
                 e^{\frac{1}{12a+1}+\frac{1}{12b+1}-\frac{1}{12(a+b)}} \\
  ~&\geq~ \frac{1}{(a+b+1)} \sqrt{\frac{2\pi ab}{a+b}} e^{-(a+b)H(\hat\theta)} e^{0+0-\frac{1}{12(1+1)}}
  ~\geq~ \frac{1.6\sqrt{ab}}{(a+b)^{3/2}}\cdot e^{-(a+b)H(\hat\theta)}
\end{align*}
Combining both bounds we get
\begin{align*}
  I_x(a+1,b+1) ~\leq~ x\cdot G(x;a+1,b+1) 
  ~~~\text{where}~~~ G(x;a+1,b+1) ~&:=~ \fr{(a+b)^{3/2}}{1.6\sqrt{ab}} e^{-(a+b)\KL(\hat\theta||x)}
\end{align*}
Hence with $a+1=k+\alpha$ and $b+1=n-k+\beta$, 
we can upper bound $P[|\theta-\mu|\geq\eps|k;n,\alpha,\beta]$ via
\begin{align*}
  P[\theta\leq\mu-\eps|k;n,\alpha,\beta] ~&=~ I_{\mu-\eps}(k+\alpha,n-k+\beta) \\
  P[\theta\geq\mu+\eps|k;n,\alpha,\beta] ~&=~ I_{1-\mu-\eps}(n-k+\beta,k+\alpha)
\end{align*}
Note that $\mu=\hat\theta$ and hence $\KL(\hat\theta||\mu-\eps)\geq 2\eps^2$,
and similarly $\KL(1-\hat\theta||1-\mu-\eps)\geq 2\eps^2$ by Pinsker's inequality.
Together this implies
\begin{equation}\label{eq:betaboundk}
  P[|\theta-\mu|\geq\eps|k;n,\alpha,\beta] ~\leq~ \fr{1.25(n+\alpha+\beta-2)^{3/2}}{\sqrt{(k+\alpha-1)(n-k+\beta-1)}} e^{-2(n+\alpha+\beta-2)\eps^2}
\end{equation}
provided $k+\alpha\geq 2$ and $n-k+\beta\geq 2$.
Setting $\alpha = \tfrac{1}{2}$ and $\beta = \tfrac{1}{2}$ yields the concentration result for the KT-estimator given in \cref{lem:kt_concentration}. 

\paragraph{Uniform/expected tail bound.} 
The $\sqrt\cdot$ denominator in \cref{eq:betaboundk} is lower-bounded by 1,
so if we drop it, we get an upper bound that is independent of $k$:
\begin{equation}\label{eq:betabound}
  P[|\theta-\mu|\geq\eps|k;n,\alpha,\beta] ~\leq~ 1.25(n'-2)^{3/2} e^{-2(n'-2)\eps^2}
  ~~~\text{where~~~} n'~:=~ n+\alpha+\beta
\end{equation}
Comparing the exponential to the asymptotic expression \cref{eq:betaasympt},
we see that the bound is asymptotically quite tight for $\mu\approx k/n\approx 1/2$,
but quite loose otherwise.
This makes it plausible that \cref{eq:betabound} also holds for 
$k+\alpha<2$ and $n-k+\beta<2$. If so, it also holds in-expectation over $k$,
i.e.\ when dropping $|k$ from the left hand side.
A similar bound can be derived more easily as follows:
First we define the random version 
$M:=(X_+\!+\alpha)/n'$ of $\mu$ by replacing $k$ by $X_+$ in $\mu$,
and derive a (frequentist) bound conditioned on $\theta$:
\begin{align*}
  & P[\theta-M\geq\eps|\theta]
  ~=~ P[M\leq\theta-\eps|\theta]
  ~\stackrel{(a)}=~ P[X_+\leq n\theta-\gamma|\theta]
  ~\smash{\stackrel{(b)}\leq}~ e^{-2\gamma^2/n} \\
  ~&\leq~ e^{-2\gamma^2/n'}
  ~\stackrel{(c)}\leq~ e^{-2(n'\eps-\beta)^2/n'}
  ~=~ B\,e^{-2n'\eps^2} ~~~\text{with}~~~ B := e^{4\eps\beta-2\beta^2/n'} \leq e^{4\eps\beta}
\end{align*}
where we used/assumed $\gamma:=n'\eps-(\alpha+\beta)\theta+\alpha\geq n'\eps-\beta\geq 0$ in $(a,c)$,
and Hoeffding's bound in $(b)$.
Note that for $n'\eps\leq\beta$ $\exists\theta:\gamma=0$ so the bound is vacuous.
The dual 
$$
  P[\theta-M\leq-\eps|\theta] ~\leq~ A\,e^{-2n'\eps^2} ~~~\text{with}~~~ A := e^{4\eps\alpha-2\alpha^2/n'} \leq e^{4\eps\alpha}
$$
is derived in the same way. Since the bounds are independent of $\theta$,
they also hold in $\theta$-expectation. So together we get
\begin{align*}
  & P[|\theta-M|\geq\eps] ~\leq~ (A+B) e^{-2n'\eps^2} ~\leq~ 2 e^{-2n'(\eps-\eps_0)^2} 
  ~~~\text{for}~~~ n' \geq \max\{\alpha,\beta\}/\eps \\
  & \text{where}~~~\eps_0 ~:=~ \smash{\sup_\theta}|\E[M]-\theta| ~=~ \smash{\sup_\theta}|\alpha-\theta(\alpha+\beta)|/n' ~=~ \max\{\alpha,\beta\}/n'
\end{align*}
is the maximal bias of $M$ relative to $\theta$.
This bound is similar to \cref{eq:betabound}.

\section{Calculating the Active Segments}
\label{sec:active_segments} 

Given a time index $t\leq 2^D$, the active segments can be efficiently determined from the binary representation of $t-1$.
To give some intuition, consider \cref{fig:active_segments}, which shows a complete binary tree of depth $D=2$, with each node describing a possible segment and left/right edges labelled 0/1 respectively.
Consider the set of active segments corresponding to time $t=3$.
In this case the binary representation of $t-1=2$ is $10_2$; reading the bitstring left to right, this defines a path down the tree starting from the root node and continuing downwards, with each bit determining which edge (0 for left, 1 for right) to follow. 
The nodes we encounter are the active segments, in this case $\{ (1,4), (3,4), (3,3) \}$.

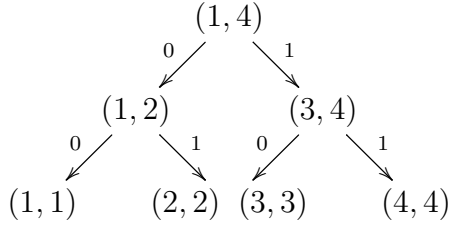
\begin{figure}
	\centerline{
		\xymatrix @ur {
			(1,1)   &(1,2) \ar[l]_0 \ar[d]^1 & (1,4) \ar[l]_0 \ar[d]^1 \\
			& (2,2) \;\; (3,3) & (3,4) \ar[l]_0 \ar[d]^1 \\
			&  & (4,4)
		}
	}
	\vspace{-6.0em}
	\caption{A labelled binary partition tree of depth 2.}\label{fig:active_segments}
\end{figure}

This observation leads to the following algorithm for computing {\sc ActiveSegments}$_D(t)$ for $t \leq 2^D$:
\begin{enumerate}\parskip=0ex\parsep=0ex\itemsep=0ex
    \item Set $b_{1:D}$ to the $D$-digit binary representation of $t-1$.\\
    \item Return $f(b_{1:D},1)$
\end{enumerate}
where $f$ is recursively defined as:
\begin{equation*}
f(b_{1:k},o) := \{o, o+2^k - 1\} \bigcup 
\begin{cases}
 \{  \} & \text{if}~~ b_{1:k} = \epstr;\\
 f(b_{2:k},o) & \text{if}~~ b_1 = 0;\\
 f(b_{2:k},o+2^{k-1}) & \text{if}~~ b_1 = 1.
\end{cases}
\end{equation*}

\clearpage

\section{Detailed Results}
\label{sec:detailed_results}
\vspace{-2em}
\begin{table}[!h] 
\caption{$T=10^6$, $|\gA|$=2, varying Geometric Change-point Rates $p$}
\vspace{-.5em}
\begin{center}
\small
\begin{tabular}{ | c | c | c | c | c |}
\hline
\bf{Algorithm} & $p=0.01$ & $p=0.001$ & $p=0.0001$ & $p=0.00001$ \\
\hline
Uniform & 16682.80 $\pm$ 112 & 16796.48 $\pm$ 349 & 16349.89 $\pm$ 1135 & 16373.79 $\pm$ 1860\\
Constant & 16831.00 $\pm$ 211 & 16955.33 $\pm$ 630 & 14636.78 $\pm$ 1845 & 16983.52 $\pm$ 3789\\
UCB & 12529.63 $\pm$ 308 & 12073.45 $\pm$ 633 & 8140.44 $\pm$ 1165 & 1445.66 $\pm$ 702\\
TS & 16115.16 $\pm$ 175 & 14689.31 $\pm$ 584 & 9766.88 $\pm$ 1328 & 2222.04 $\pm$ 921\\
SWUCB & 8696.69 $\pm$ 112 & 3556.60 $\pm$ 47 & 2882.61 $\pm$ 71 & 2906.37 $\pm$ 141\\
MASTER & 10788.93 $\pm$ 288 & 8778.16 $\pm$ 483 & 4511.88 $\pm$ 606 & 1137.33 $\pm$ 194\\
\hline
ActivePTW & 4872.67 $\pm$ 43 & 1625.44 $\pm$ 51 & 453.85 $\pm$ 67 & 189.19 $\pm$ 111\\
ParanoidPTW & 5288.69 $\pm$ 45 & 1936.64 $\pm$ 48 & 706.17 $\pm$ 57 & 416.07 $\pm$ 59\\
\hline
\end{tabular}
\end{center}
\label{tbl:geo1}
\end{table}


\vspace{-3em}
\begin{table}[h] 
\caption{$T=10^6$, $|\gA|$=5, varying Geometric Change-point Rates $p$}
\vspace{-.5em}
\begin{center}
\small
\begin{tabular}{ | c | c | c | c | c |}
\hline
\bf{Algorithm} & $p=0.01$ & $p=0.001$ & $p=0.0001$ & $p=0.00001$ \\
\hline
Uniform & 33319.00 $\pm$ 107 & 33404.95 $\pm$ 338 & 33159.95 $\pm$ 890 & 33303.03 $\pm$ 2006\\
Constant & 33409.47 $\pm$ 243 & 33191.28 $\pm$ 726 & 33825.04 $\pm$ 2149 & 37767.84 $\pm$ 4179\\
UCB & 15623.75 $\pm$ 238 & 14094.51 $\pm$ 784 & 10144.91 $\pm$ 1202 & 1702.98 $\pm$ 453\\
TS & 30861.35 $\pm$ 196 & 24655.39 $\pm$ 751 &  14350.63 $\pm$ 1132 & 3098.01 $\pm$ 841\\
SWUCB & 13512.51 $\pm$ 76 & 10053.85 $\pm$ 52 & 9738.46 $\pm$ 134 & 9564.74 $\pm$ 287\\
MASTER & 14132.31 $\pm$ 226 & 11498.56 $\pm$ 519 & 6630.28 $\pm$ 480 & 2771.50 $\pm$ 194\\
\hline
ActivePTW & 11688.19 $\pm$ 63 & 3710.23 $\pm$ 71 & 1038.22 $\pm$ 100 & 255.58 $\pm$ 92\\
ParanoidPTW & 12404.06 $\pm$ 68 & 4333.63 $\pm$ 69 & 1580.39 $\pm$ 73 & 813.64 $\pm$ 56\\
\hline
\end{tabular}
\end{center}
\label{tbl:geo2}
\end{table}


\vspace{-3em}
\begin{table}[h!] 
\caption{$T=10^6$, $|\gA|$=10, varying Geometric Change-point Rates $p$}
\vspace{-.5em}
\begin{center}
\small
\begin{tabular}{ | c | c | c | c | c |}
\hline
\bf{Algorithm} & $p=0.01$ & $p=0.001$ & $p=0.0001$ & $p=0.00001$\\
\hline
Uniform & 40855.62 $\pm$ 88 & 40638.48 $\pm$ 239 & 41087.54 $\pm$ 793 & 41243.84 $\pm$ 1625\\
Constant & 40910.90 $\pm$ 270 & 40269.33 $\pm$ 751 & 39285.85 $\pm$ 2336 & 42997.03 $\pm$ 3882\\
UCB & 15211.67 $\pm$ 178 & 11247.85 $\pm$ 510 & 7834.39 $\pm$ 722 & 1827.94 $\pm$ 426\\
TS & 36997.58 $\pm$ 137 & 26843.46 $\pm$ 445 & 14210.46 $\pm$ 853 & 2969.75 $\pm$ 697\\
SWUCB & 18909.91 $\pm$ 51 & 17477.58 $\pm$ 79 & 17412.77 $\pm$ 219 & 17350.04 $\pm$ 454\\
MASTER & 14855.74 $\pm$ 139 & 10914.08 $\pm$ 293 & 7512.62 $\pm$ 374 & 4448.03 $\pm$ 227\\
\hline
ActivePTW & 16559.30 $\pm$ 74 & 5303.99 $\pm$ 83 & 1384.77 $\pm$ 85 & 345.48 $\pm$ 115\\
ParanoidPTW & 17360.79 $\pm$ 72 & 6041.04 $\pm$ 76 & 2074.52 $\pm$ 71 & 1049.65 $\pm$ 67 \\
\hline
\end{tabular}
\end{center}
\label{tbl:geo3}
\end{table}

\vspace{-3em}
\begin{table}[H]
\caption{$T=10^6$, $|\gA|$=50, varying Geometric Change-point Rates $p$}
\vspace{-.5em}
\begin{center}
\small
\begin{tabular}{ | c | c | c | c | c |}
\hline
\bf{Algorithm} & $p=0.01$ & $p=0.001$ & $p=0.0001$ & $p=0.00001$ \\
\hline
Uniform & 48048.92 $\pm$ 44 & 47985.77 $\pm$ 121 & 48130.98 $\pm$ 362 & 48587.99 $\pm$ 759\\
Constant & 47928.04 $\pm$ 240 & 48347.07 $\pm$ 703 & 48385.84 $\pm$ 2204 & 51621.47 $\pm$ 3186\\
UCB & 15557.03 $\pm$ 81 & 7722.83 $\pm$ 85 & 5517.39 $\pm$ 194 & 3782.69 $\pm$ 181\\
TS & 42613.60 $\pm$ 117 & 26953.07 $\pm$ 292 & 10797.99 $\pm$ 540 & 2280.35 $\pm$ 474\\
SWUCB & 36772.11 $\pm$ 49 & 35568.11 $\pm$ 108 & 35595.94 $\pm$ 329 & 36058.06 $\pm$ 640\\
MASTER & 19597.48 $\pm$ 210 & 14193.67 $\pm$ 342 & 12507.24 $\pm$ 376 & 11449.53 $\pm$ 416\\
\hline
ActivePTW & 28429.17 $\pm$ 82 & 10837.19 $\pm$ 139 & 2849.74 $\pm$ 143 & 624.28 $\pm$ 114\\
ParanoidPTW & 29106.90 $\pm$ 87 & 11711.55 $\pm$ 136 & 3674.19 $\pm$ 140 & 1440.15 $\pm$ 83\\
\hline
\end{tabular}
\end{center}
\label{tbl:geo4}
\end{table}


\end{document}

%% file: math_commands.tex
\usepackage{amsmath,amsfonts,bm}

\def\eqref#1{equation~\ref{#1}}
\def\Eqref#1{Equation~\ref{#1}}

\def\1{\bm{1}}

\def\eps{\varepsilon}            
\def\epstr{\epsilon}            
\def\fr#1#2{{\textstyle\frac{#1}{#2}}} 
\def\frs#1#2{{^{#1}\!/\!_{#2}}} 

\DeclareMathAlphabet{\mathsfit}{\encodingdefault}{\sfdefault}{m}{sl}
\SetMathAlphabet{\mathsfit}{bold}{\encodingdefault}{\sfdefault}{bx}{n}

\def\gA{{\mathcal{A}}}

\def\gC{{\mathcal{C}}}

\def\gE{{\mathcal{E}}}

\def\gH{{\mathcal{H}}}

\def\gL{{\mathcal{L}}}
\def\gM{{\mathcal{M}}}

\def\gO{{\mathcal{O}}}
\def\gP{{\mathcal{P}}}

\def\gR{{\mathcal{R}}}
\def\gS{{\mathcal{S}}}

\def\gU{{\mathcal{U}}}

\def\gX{{\mathcal{X}}}

\newcommand{\E}{\mathbb{E}}

\newcommand{\KL}{D_{\text{KL}}}

\DeclareMathOperator*{\argmax}{arg\,max}
\DeclareMathOperator*{\argmin}{arg\,min}

%% file: main.bbl
\begin{thebibliography}{33}
\providecommand{\natexlab}[1]{#1}
\providecommand{\url}[1]{\texttt{#1}}
\expandafter\ifx\csname urlstyle\endcsname\relax
  \providecommand{\doi}[1]{doi: #1}\else
  \providecommand{\doi}{doi: \begingroup \urlstyle{rm}\Url}\fi

\bibitem[Auer et~al.(2002)Auer, Cesa-Bianchi, and Fischer]{Auer2002}
Peter Auer, Nicol{\`o} Cesa-Bianchi, and Paul Fischer.
\newblock Finite-time analysis of the multiarmed bandit problem.
\newblock \emph{Machine Learning}, 47:\penalty0 235--256, 2002.
\newblock URL \url{https://api.semanticscholar.org/CorpusID:207609497}.

\bibitem[Chapelle and Li(2011)]{Chapelle2011}
Olivier Chapelle and Lihong Li.
\newblock An empirical evaluation of thompson sampling.
\newblock In \emph{Neural Information Processing Systems}, 2011.
\newblock URL \url{https://api.semanticscholar.org/CorpusID:6002655}.

\bibitem[Cutkosky(2020)]{cutkosky20a}
Ashok Cutkosky.
\newblock Parameter-free, dynamic, and strongly-adaptive online learning.
\newblock In \emph{Proceedings of the 37th International Conference on Machine
  Learning}, pages 2250--2259, 2020.
\newblock URL \url{https://proceedings.mlr.press/v119/cutkosky20a.html}.

\bibitem[Daniely et~al.(2015)Daniely, Gonen, and Shalev-Shwartz]{daniely15}
Amit Daniely, Alon Gonen, and Shai Shalev-Shwartz.
\newblock Strongly adaptive online learning.
\newblock In \emph{Proceedings of the 32nd International Conference on Machine
  Learning}, pages 1405--1411, 2015.

\bibitem[Garivier and Capp{\'e}(2011)]{Garivier2011}
Aur{\'e}lien Garivier and Olivier Capp{\'e}.
\newblock The kl-ucb algorithm for bounded stochastic bandits and beyond.
\newblock \emph{ArXiv}, abs/1102.2490, 2011.
\newblock URL \url{https://api.semanticscholar.org/CorpusID:6345726}.

\bibitem[Garivier and Moulines(2008)]{Garivier2008}
Aur{\'e}lien Garivier and {\'E}ric Moulines.
\newblock On upper-confidence bound policies for non-stationary bandit
  problems.
\newblock \emph{arXiv: Statistics Theory}, 2008.
\newblock URL \url{https://api.semanticscholar.org/CorpusID:15513611}.

\bibitem[Gy{\"o}rgy et~al.(2012)Gy{\"o}rgy, Linder, and Lugosi]{Gyorgy2012}
Andr{\'a}s Gy{\"o}rgy, Tam{\'a}s Linder, and G{\'a}bor Lugosi.
\newblock Efficient tracking of large classes of experts.
\newblock \emph{IEEE Transactions on Information Theory}, 58:\penalty0
  6709--6725, 2012.
\newblock URL \url{https://api.semanticscholar.org/CorpusID:5892094}.

\bibitem[Hazan and Seshadhri(2009)]{Hazan2009}
Elad Hazan and Comandur Seshadhri.
\newblock Efficient learning algorithms for changing environments.
\newblock In \emph{International Conference on Machine Learning}, 2009.
\newblock URL \url{https://api.semanticscholar.org/CorpusID:6083044}.

\bibitem[Hutter et~al.(2024)Hutter, Quarel, and Catt]{Hutter:24uaibook2}
Marcus Hutter, David Quarel, and Elliot Catt.
\newblock \emph{An Introduction to Universal Artificial Intelligence}.
\newblock Chapman \& Hall/CRC Artificial Intelligence and Robotics Series.
  Taylor and Francis, 2024.

\bibitem[Kaufmann et~al.(2012)Kaufmann, Korda, and Munos]{kaufmann2012thompson}
Emilie Kaufmann, Nathaniel Korda, and Rémi Munos.
\newblock Thompson sampling: An asymptotically optimal finite time analysis,
  2012.

\bibitem[Lattimore and Hutter(2014)]{Lattimore2014BayesianRL}
Tor Lattimore and Marcus Hutter.
\newblock Bayesian reinforcement learning with exploration.
\newblock In \emph{International Conference on Algorithmic Learning Theory},
  2014.
\newblock URL \url{https://api.semanticscholar.org/CorpusID:18388290}.

\bibitem[Lattimore and Szepesvari(2020)]{Lattimore2020Book}
Tor Lattimore and Csaba Szepesvari.
\newblock Bandit algorithms.
\newblock In \emph{Bandit Algorithms}, 2020.
\newblock URL \url{https://api.semanticscholar.org/CorpusID:242172176}.

\bibitem[Leike et~al.(2016)Leike, Lattimore, Orseau, and Hutter]{Leike2016}
Jan Leike, Tor Lattimore, Laurent Orseau, and Marcus Hutter.
\newblock Thompson sampling is asymptotically optimal in general environments.
\newblock \emph{ArXiv}, abs/1602.07905, 2016.
\newblock URL \url{https://api.semanticscholar.org/CorpusID:5739028}.

\bibitem[Ortega and Braun(2008)]{Ortega2008}
Pedro~A. Ortega and Daniel~A. Braun.
\newblock A minimum relative entropy principle for learning and acting.
\newblock \emph{J. Artif. Intell. Res.}, 38:\penalty0 475--511, 2008.
\newblock URL \url{https://api.semanticscholar.org/CorpusID:5854579}.

\bibitem[Ortega and Braun(2009)]{Ortega2009}
Pedro~A. Ortega and Daniel~A. Braun.
\newblock A bayesian rule for adaptive control based on causal interventions.
\newblock \emph{ArXiv}, abs/0911.5104, 2009.
\newblock URL \url{https://api.semanticscholar.org/CorpusID:3128221}.

\bibitem[Ortega and Braun(2010)]{Ortega2010AnAF}
Pedro~A. Ortega and Daniel~A. Braun.
\newblock An axiomatic formalization of bounded rationality based on a
  utility-information equivalence.
\newblock \emph{ArXiv}, abs/1007.0940, 2010.
\newblock URL \url{https://api.semanticscholar.org/CorpusID:2247227}.

\bibitem[Ortega and Braun(2012)]{Ortega2012AdaptiveCO}
Pedro~A. Ortega and Daniel~A. Braun.
\newblock Adaptive coding of actions and observations.
\newblock In \emph{Neural Information Processing Systems}, 2012.
\newblock URL \url{https://api.semanticscholar.org/CorpusID:14150250}.

\bibitem[Ortega and Braun(2013)]{ortega14}
Pedro~A. Ortega and Daniel~A. Braun.
\newblock Generalized thompson sampling for sequential decision-making and
  causal inference.
\newblock \emph{CoRR}, abs/1303.4431, 2013.
\newblock URL \url{http://arxiv.org/abs/1303.4431}.

\bibitem[Ortega et~al.(2021)Ortega, Kunesch, Delétang, Genewein, Grau-Moya,
  Veness, Buchli, Degrave, Piot, Perolat, Everitt, Tallec, Parisotto, Erez,
  Chen, Reed, Hutter, de~Freitas, and Legg]{ortega2021shaking}
Pedro~A. Ortega, Markus Kunesch, Grégoire Delétang, Tim Genewein, Jordi
  Grau-Moya, Joel Veness, Jonas Buchli, Jonas Degrave, Bilal Piot, Julien
  Perolat, Tom Everitt, Corentin Tallec, Emilio Parisotto, Tom Erez, Yutian
  Chen, Scott Reed, Marcus Hutter, Nando de~Freitas, and Shane Legg.
\newblock Shaking the foundations: delusions in sequence models for interaction
  and control, 2021.

\bibitem[Russo et~al.(2017)Russo, Roy, Kazerouni, and Osband]{Russo2017}
Daniel Russo, Benjamin~Van Roy, Abbas Kazerouni, and Ian Osband.
\newblock A tutorial on thompson sampling.
\newblock \emph{ArXiv}, abs/1707.02038, 2017.
\newblock URL \url{https://api.semanticscholar.org/CorpusID:3929917}.

\bibitem[Shamir and Merhav(1998)]{Shamir1998}
Gil~I. Shamir and Neri Merhav.
\newblock Low-complexity sequential lossless coding for piecewise-stationary
  memoryless sources.
\newblock \emph{IEEE Trans. Inf. Theory}, 45:\penalty0 1498--1519, 1998.
\newblock URL \url{https://api.semanticscholar.org/CorpusID:2638909}.

\bibitem[Strens(2000)]{Strens2000}
Malcolm J.~A. Strens.
\newblock A bayesian framework for reinforcement learning.
\newblock In \emph{International Conference on Machine Learning}, 2000.
\newblock URL \url{https://api.semanticscholar.org/CorpusID:983224}.

\bibitem[Thompson(1933)]{Thompson1933}
William~R. Thompson.
\newblock On the likelihood that one unknown probability exceeds another in
  view of the evidence of two samples.
\newblock \emph{Biometrika}, 25:\penalty0 285--294, 1933.
\newblock URL \url{https://api.semanticscholar.org/CorpusID:120462794}.

\bibitem[Tjalkens et~al.(1993)Tjalkens, Shtarkov, and
  Willems]{Tjalkens1993SequentialWA}
Tjalling~J. Tjalkens, Yuri~M. Shtarkov, and Frans M.~J. Willems.
\newblock Sequential weighting algorithms for multi-alphabet sources.
\newblock In \emph{6th Joint Swedish-Russian Int. Worksh. Inform. Theory},
  1993.
\newblock URL \url{https://api.semanticscholar.org/CorpusID:14494100}.

\bibitem[Trov{\`o} et~al.(2020)Trov{\`o}, Paladino, Restelli, and
  Gatti]{Trovo2020}
Francesco Trov{\`o}, Stefano Paladino, Marcello Restelli, and Nicola Gatti.
\newblock Sliding-window thompson sampling for non-stationary settings.
\newblock \emph{J. Artif. Intell. Res.}, 68:\penalty0 311--364, 2020.
\newblock URL \url{https://api.semanticscholar.org/CorpusID:219178453}.

\bibitem[Veness et~al.(2013)Veness, White, Bowling, and György]{veness13}
Joel Veness, Martha White, Michael Bowling, and András György.
\newblock Partition tree weighting.
\newblock In \emph{2013 Data Compression Conference}, pages 321--330, 2013.
\newblock \doi{10.1109/DCC.2013.40}.

\bibitem[Wei and Luo(2021)]{Wei2021}
Chen-Yu Wei and Haipeng Luo.
\newblock Non-stationary reinforcement learning without prior knowledge: An
  optimal black-box approach.
\newblock \emph{ArXiv}, abs/2102.05406, 2021.
\newblock URL \url{https://api.semanticscholar.org/CorpusID:231861770}.

\bibitem[Wiener(1950)]{Wiener1950}
Norbert Wiener.
\newblock The human use of human beings.
\newblock In \emph{The Human Use of Human Beings}, 1950.
\newblock URL \url{https://api.semanticscholar.org/CorpusID:263489887}.

\bibitem[Willems et~al.(1995)Willems, Shtarkov, and Tjalkens]{Willems95}
F.M.J. Willems, Y.M. Shtarkov, and T.J. Tjalkens.
\newblock The context-tree weighting method: basic properties.
\newblock \emph{IEEE Transactions on Information Theory}, 41\penalty0
  (3):\penalty0 653--664, 1995.
\newblock \doi{10.1109/18.382012}.

\bibitem[Willems and Krom(1997)]{Willems1997}
Fmj~Frans Willems and Melven~R. Krom.
\newblock Live-and-die coding for binary piecewise i.i.d. sources.
\newblock \emph{Proceedings of IEEE International Symposium on Information
  Theory}, pages 68--, 1997.
\newblock URL \url{https://api.semanticscholar.org/CorpusID:123630162}.

\bibitem[Willems(1996)]{Willems1996}
Frans M.~J. Willems.
\newblock Coding for a binary independent piecewise-identically-distributed
  source.
\newblock \emph{IEEE Trans. Inf. Theory}, 42:\penalty0 2210--2217, 1996.
\newblock URL \url{https://api.semanticscholar.org/CorpusID:21920343}.

\bibitem[Witten et~al.(1987)Witten, Neal, and Cleary]{Witten87}
Ian~H. Witten, Radford~M. Neal, and John~G. Cleary.
\newblock Arithmetic coding for data compression.
\newblock \emph{Commun. ACM}, 30\penalty0 (6):\penalty0 520–540, June 1987.
\newblock ISSN 0001-0782.
\newblock \doi{10.1145/214762.214771}.
\newblock URL \url{https://doi.org/10.1145/214762.214771}.

\bibitem[Zhao et~al.(2024)Zhao, Zhang, Zhang, and Zhou]{zhao24}
Peng Zhao, Yu-Jie Zhang, Lijun Zhang, and Zhi-Hua Zhou.
\newblock Adaptivity and non-stationarity: Problem-dependent dynamic regret for
  online convex optimization.
\newblock \emph{Journal of Machine Learning Research}, 25\penalty0
  (98):\penalty0 1--52, 2024.
\newblock URL \url{http://jmlr.org/papers/v25/21-0748.html}.

\end{thebibliography}
